\documentclass[12pt]{iopart}

\usepackage[utf8]{inputenc} % allow utf-8 input
\usepackage[T1]{fontenc}    % use 8-bit T1 fonts
\usepackage{hyperref}       % hyper links
\usepackage{url}            % simple URL typesetting
\usepackage{booktabs}       % professional-quality tables
\usepackage{amsthm}
\usepackage{amsmath,amsfonts,bm}

\def\eqref#1{equation~\ref{#1}}
\def\1{\bm{1}}

\def\vphi{{\bm{\phi}}}

\def\dw{{\bm{\delta w}}}
\def\vtheta{{\bm{\theta}}}

\def\ve{{\bm{e}}}

\def\vg{{\bm{g}}}

\def\vv{{\bm{v}}}
\def\vw{{\bm{w}}}
\def\vx{{\bm{x}}}
\def\vy{{\bm{y}}}

\def\mB{{\bm{B}}}

\def\mH{{\bm{H}}}
\def\mI{{\bm{I}}}

\def\mU{{\bm{U}}}

\def\mX{{\bm{X}}}
\def\mY{{\bm{Y}}}

\def\meps{{\bm{\epsilon}}}

\DeclareMathAlphabet{\mathsfit}{\encodingdefault}{\sfdefault}{m}{sl}
\SetMathAlphabet{\mathsfit}{bold}{\encodingdefault}{\sfdefault}{bx}{n}

\DeclareMathOperator*{\argmin}{arg\,min}

\newtheorem{theorem}{Theorem}

\usepackage{xcolor}
\usepackage{subcaption}
\usepackage{graphicx}
\usepackage{enumitem}
\usepackage{algorithm}
\usepackage{algorithmic}
\usepackage{amssymb}
\usepackage{comment}
\usepackage{wrapfig}

\newcommand{\reff}{\mathcal{R}_\text{est-curv}}
\newcommand{\reflong}{estimated curvature learning rate ratio}

\definecolor{brown}{HTML}{994F00} 
\definecolor{navy}{HTML}{006CD1} 

\usepackage[textsize=footnotesize]{todonotes}

\newcommand{\nick}[1]{\textcolor{black}{#1}}
\newtheorem{lemma}[theorem]{Lemma}
\newcommand{\Pdata}{\mathbb{P}_{\text{data}}}
\newcommand{\latestEdits}[1]{{\color{black}{#1}}}

\usepackage[normalem]{ulem}
\useunder{\uline}{\ul}{}

\begin{document}

% Possible titles:

% \title{A RMT Approach to the Adaptive Gradient Methods in Deep Neural Networks}
% \title{Adaptive gradient methods in deep neural networks: insights from random matrix theory}
\title{A Random Matrix Theory Approach to Damping in Deep Learning}
\author[]{Diego Granziol, Huawei AI Theory}
\author[]{Nicholas Baskerville, Bristol University}

\begin{abstract}
We conjecture that the inherent difference in generalisation between adaptive and non-adaptive gradient methods in deep learning stems from the increased estimation noise in the flattest directions of the true loss surface. We demonstrate that typical schedules used for adaptive methods (with low numerical stability or damping constants) serve to bias relative movement towards flat directions relative to sharp directions, effectively amplifying the noise-to-signal ratio and harming generalisation. We further demonstrate that the numerical damping constant used in these methods can be decomposed into a learning rate reduction and linear shrinkage of the estimated curvature matrix. We then demonstrate significant generalisation improvements by increasing the shrinkage coefficient, closing the generalisation gap entirely in both logistic regression and several deep neural network experiments. Extending this line further, we develop a novel random matrix theory based damping learner for second order optimisers inspired by linear shrinkage estimation. We experimentally demonstrate our learner to be very insensitive to the initialised value and to allow for extremely fast convergence in conjunction with continued stable training and competitive generalisation. We also find that our derived method works well with adaptive gradient methods such as Adam.

\end{abstract}

\section{Introduction}\label{sec:intro}
The success of deep neural networks across a wide variety of tasks, from speech recognition to image classification,
has drawn wide-ranging interest in their optimisation and their ability to generalise to unseen data. Optimisation is the process of reducing the value of the neural network loss on its training set. Generalisation refers to the loss on an unseen held out test set. The loss of a neural network is a scalar function of the network free parameters (the weights of the neural network) that measures how well the network is performing on the data at hand. It can be defined on a single or multiple examples (known as a batch). Loss functions are by definition greater than or equal to zero. For example, a neural network's output may be the probability of a particular image class label being correct. In this case, where we have the true label of the data item, we can use the discrete cross entropy loss, which has roots in information theory \cite{cover2012elements}. The discrete cross entropy loss is zero if and only if we predict the true label with probability $1$ and non zero otherwise. Hence we drive the neural network to learn not only the true class, but to predict the true class with probability $1$, i.e. to be confident about the correct predictions. Alternative losses in the regression context could be the well known square loss, along with many others, such as the Hinge loss \cite{bishop2006pattern}.

Due to enormous dimensionality of the weight space, where billions of parameters are the norm, more effective measures than random search must be employed to reduce the loss from a random initialisation of the weights. A very simple but highly effective method, which underlies the basis of modern machine learning optimisation is stochastic gradient descent (SGD). In its simplest form, we simply take the gradient of the neural network batch loss and follow the method of steepest descent into a local minima. For a full discussion on stochastic gradient optimisation and associated convergence proofs we recommend \cite{nesterov2013introductory,kushner2003stochastic}

% \section{Adaptive Gradient Optimisers}
\medskip
There have been many amendments to stochastic gradient descent, including the use of momentum \cite{nesterov2013introductory}. A particularly fruitful area of optimisation research which has found its way into Deep Learning are \textit{adaptive gradient optimisers}. Adaptive gradient methods alter the per-parameter learning rate depending on historical gradient information, which leads to significantly faster convergence of the training loss than non adaptive methods, such as stochastic gradient descent (SGD) with momentum \cite{nesterov2013introductory}. Popular examples include Adam \cite{kingma2014adam}, AdaDelta \cite{zeiler2012adadelta} and RMSprop \cite{tieleman2012lecture}. 

\medskip
However, for practical applications the final results on a held-out test set are more important than the training performance. 
%The difference between training and test performance is known as the \textit{generalisation gap}. 
For many image and language problems of interest, the test set performance of adaptive gradient methods is significantly worse than SGD \cite{wilson2017marginal}---a phenomenon that we refer to as the \textit{adaptive generalisation gap}. As a consequence of this effect, many state-of-the-art models, especially for image classification datasets such as CIFAR \cite{yun2019cutmix} and ImageNet \cite{xie2019selftraining,cubuk2019randaugment}, are still trained using SGD with momentum. Although less widely used, another class of adaptive methods which suffer from the same phenomenon \cite{tornstad2020evaluating} are \emph{stochastic second order methods}, which seek to alter the learning rate along the eigenvectors of the Hessian of the loss function. KFAC~\cite{martens2015optimizing} uses a Kroenecker factored approximation of the Fisher information matrix (which can be seen as a positive definite approximation to the Hessian \cite{martens2014new}). Other methods use Hessian--vector products \cite{dauphin2014identifying,martens2010deep} in conjunction with Lanczos methods and conjugate gradients \cite{meurant2006lanczos}. 
All second order and adaptive gradient methods, are endowed with an extra hyper-parameter called the damping or numerical stability co-efficient respectively. This parameter limits the maximal learning rate along the eigenvectors or unit vectors in the parameter space respectively and is typically set to a very small value by practitioners.

\medskip
In this paper we argue that adaptive methods in their typical implementations with small damping/numerical stability co-efficients, are over-confident in their updates in the flattest directions of the loss. We show that this is sub-optimal in terms of optimising the true loss and hence harms generalisation performance. We demonstrate this empirically in an online convex example, where we actively perturb the sharp directions, reducing generalisation without impacting training. We also demonstrate this implicitly for large neural networks by altering the damping/stability constant, which we show alters the effective learning rate ratio between the sharp and flat directions. 

\medskip
Given their widespread adoption, understanding the adaptive generalisation gap has significant implications. In this work we show that altering the numerical stability constant in Adam \cite{choi2019empirical} can be interpreted as applying linear shrinkage estimation to the noisy Hessian estimate enabling the reduction of the mean squared error~\cite{bun2016rotational} in the estimation of the true loss Hessian. We develop this idea further to produce a novel optimal and highly effective adaptive damping scheme to automatically tune the damping and numerical stability constant for KFAC and Adam respectively. 

\section{Background on loss surfaces and generalisation}
We view a neural network (or any supervised machine learning model) as a prediction function  $h(\cdot;\cdot):\mathbb{R}^{d_{x}}\times \mathbb{R}^{P} \rightarrow \mathbb{R}^{d_{y}}$, where $\mathbb{R}^P$ is the space of parameters of the networks (i.e. its weights and biases). A single data point is an element of $\mathbb{R}^{d_x}$ and its label, which could be continuous or discrete, is an element of $\mathbb{R}^{d_y}$. Viewing $\vw\in\mathbb{R}^P$ as parameters of $h$, we have a parametrised family of functions $\mathbb{R}^{d_x}\rightarrow\mathbb{R}^{d_y}$, namely $\mathcal{H}:= \{h(\cdot;\vw):\vw \in \mathbb{R}^{P} \}$. To train the network, i.e. optimise the parameters for a given task, we define a \emph{loss function} $\ell(\vy, \hat{\vy}): \mathbb{R}^{d_{y}} \times \mathbb{R}^{d_{y}} \rightarrow \mathbb{R}$. The objective is to minimise $\ell$ over a given data distribution. More precisely, let $\Pdata$ be some probability distribution on $\mathbb{R}^{d_x}\times \mathbb{R}^{d_y}$, the \emph{data distribution}. The expectation of the loss over data distribution is known as the \emph{true loss}\footnote{The true loss is also often called the \emph{Bayes risk}.}:
\begin{equation}
	\label{eq:truerisk}
	L_{\mathrm{true}}(\vw) = \int\ell(h(\vx;\vw),\vy)d\Pdata(\vx,\vy).
\end{equation}

% The gradient with respect to the $\nabla L_{\mathrm{true}}(\vw)$  and Hessian $\mH_{\mathrm{true}}(\vw) = \nabla^{2} L_{\mathrm{true}}(\vw) \in \mathbb{R}^{P\times P}$.

In practice, given a finite dataset of size $N$, we only have access to the \emph{empirical loss} or \emph{batch loss}\footnote{The empirical loss is also known as the \emph{empirical risk}.}:
\begin{equation}
	\label{eq:emprisk}
	L_{\mathrm{emp}}(\vw) = \frac{1}{N}\sum_{i=1}^{N}\ell(h(\vx_{i};\vw),\vy_{i}),\,\, 	L_{\mathrm{batch}}(\vw) = \frac{1}{B} \sum_{i\in I_B}\ell(h(\vx_{i};\vw),\vy_{i}) 
\end{equation}
where $(\vx_i, \vy_i)$ are i.i.d. samples from $\Pdata$, $I_B\subset\{1, \ldots, N\}$ has cardinality $B$ and $B$ is typically much less that $N$.
When optimising deep neural networks in practice, the batch loss is almost always the quantity used as, amongst other reasons, the empirical loss is far too costly to evaluate given the number of training iterations that are required.
At the conclusion of some given training procedure, $\vw$ has been modified to minimise as far as possible $\ell$ on $\mathcal{D}_{\text{train}} = \{(\vx_i, \vy_i) \mid i=1,\ldots, N\}$ and so is, in general, statistically dependent on the samples in $\mathcal{D}_{\text{train}}$, thus $L_{\text{emp}}$ is no longer an unbiased estimator of $L_{\text{true}}$.
This leads to the possibility of a \emph{generalisation gap}, i.e. $L_{\text{emp}} < L_{\text{true}}$.
The true loss is the quantity that is really practically relevant and it can estimated without bias using a held out \emph{test set} which is just another finite sample from $\Pdata$ independent of $\mathcal{D}_{\text{train}}$. Note that the test loss is not free from variance and so a sufficiently large amount of the original un-partitioned dataset must be held out so that statistically significant differences can be observed.
The losses, viewed as scalar functions on $\mathbb{R}^P$ can be thought of as surfaces in $\mathbb{R}^P$ and our so known as \emph{loss surfaces}. The objective of learning is to find the lowest possible point on $L_{\text{true}}$ given only $L_{\text{batch}}$ (or $L_{\text{emp}}$). In practice this is achieved stochastic gradient optimisation methods such as stochastic gradient descent (SGD) and variants thereof. The gradients required for SGD and optimisation methods are $\nabla L_{\text{batch}}$ where derivatives are with respect the the parameters $\vw\in\mathbb{R}^P$. Similarly, second order methods also make use of the Hessian $\mH_{\text{batch}} = \nabla^2 L_{\text{batch}}$. Note that, just as $L_{\text{batch}}$ is a random estimate of $L_{\text{true}}$ corrupted by some sampling noise, so is $\nabla L_{\text{batch}}$ a noisy estimate of $\nabla L_{\text{true}}$ and similarly for the Hessian.

\subsection{Background on adaptive optimisers}
Stochastic gradient descent updates weights according to the rule
\begin{align}
    \vw_{k+1} = \vw_{k} - \alpha_k\nabla L_{\text{batch}}
\end{align}
where $\vw_k$ are the network parameters after $k$ iterations of SGD and at each iteration a different batch is used. $\alpha_k>0$ is the \emph{learning rate} which, in the simplest setting for SGD, does not depend on $k$, but in general can be varied throughout training to achieve superior optimisation and generalisation.
The general for of adaptive optimiser updates is 
\begin{align}
    \vw_{k+1} = \vw_{k} - \alpha_k\mB^{-1}\nabla L_{\text{batch}}
\end{align}
where $\mB$ is a \emph{pre-conditioning matrix}.
The essential idea of adaptive methods is to use the pre-conditioning matrix to make the geometry of $L_{\text{batch}}$ more favourable to SGD.
One approach is to take $\mB$ to be diagonal, which can be thought of as having per-parameter learning rates adapted to the local loss surface geometry.
More generally, one might seek an approximation $\mB$ to the local loss surface Hessian, effectively changing the basis of the update rule to a natural one, with per-direction learning rates.
% Alternatively, if $B = \mH_{\text{batch}}$ then the local quadratic approximation to the loss surface, i.e. the second-order term in a Taylor expansion, is isotropic in weight space. \dg{This whole section looks wrong to me...}
% What both of these approaches have in common is that they in principle allow for bigger steps (i.e. larger $\alpha_k$, as the different scales of the $\nabla L_{\text{batch}}$ in the different parameters are normalised. Indeed, a standard approach for diagonal $\mB$ is to construct a diagonal approximation to $\mH_{\text{batch}}$. Without this, $\alpha_k$ must essentially be tuned to be so small that the change of $\vw$ in the direction of the largest component of $\nabla L_{\text{batch}}$ is not too large.
For Adam \cite{kingma2014adam}, the most commonplace adaptive optimiser in the deep learning community, $\mB$ is given by the  diagonal matrix with entries $\frac{\sqrt{\langle{g}^{2}_{k}\rangle}+\epsilon}{\langle{g}_{k}\rangle}$. Here $\vg$ is the loss gradient and $\langle\cdot\rangle$ denotes an empirical exponential moving average or iterations.
% with some coefficients $\beta_{1},\beta{2}$ respectively.
In principle there is no reason why a certain parameter gradient should not be zero (or very small) and hence the inversion of $\mB$ could cause numerical issues. This is the original reason given by \cite{kingma2014adam} for the numerical stability coefficient $\epsilon$. Similarly so for KFAC for which $\mB = \sum_{i}^{P}\lambda_{i}\vphi_{i}\vphi_{i}^{T}$ where $\{\lambda_{i},\vphi_{i}\}_{i=1}^P$ are the eigenvalue, eigenvector pairs of the kronecker factored approximation to the Hessian.
Hence to each eigenvalue a small damping coefficient $\delta$ is added. Whilst for both adaptive and second order gradient methods, the numerical stability and damping coefficients are typically treated in the literature as extra nuisance parameters which are required to be non-zero but not of great theoretical or practical importance, we strongly challenge this view. In this paper we relate these coefficients to the well known linear shrinkage method in statistics. It is clear from a random matrix theory perspective, that the sub-sampling of the Hessian will lead to the creation of a noise bulk in its spectrum around the origin, precisely the region where the damping coefficient is most relevant.
We show, both experimentally and theoretically, that these coefficients should be considered as extremely important hyper-parameters whose tuning has a strong impact on generalisation. Furthermore, we provide a novel algorithm for their online estimation, which we find effective in preliminary experiments on real networks and datasets.  

\section{Previous Work}\label{sec:motivation}
\begin{figure}[h]
	\centering
	\begin{subfigure}[b]{0.49\linewidth}
		\includegraphics[trim=0cm 0 0 0,clip,width=\textwidth]{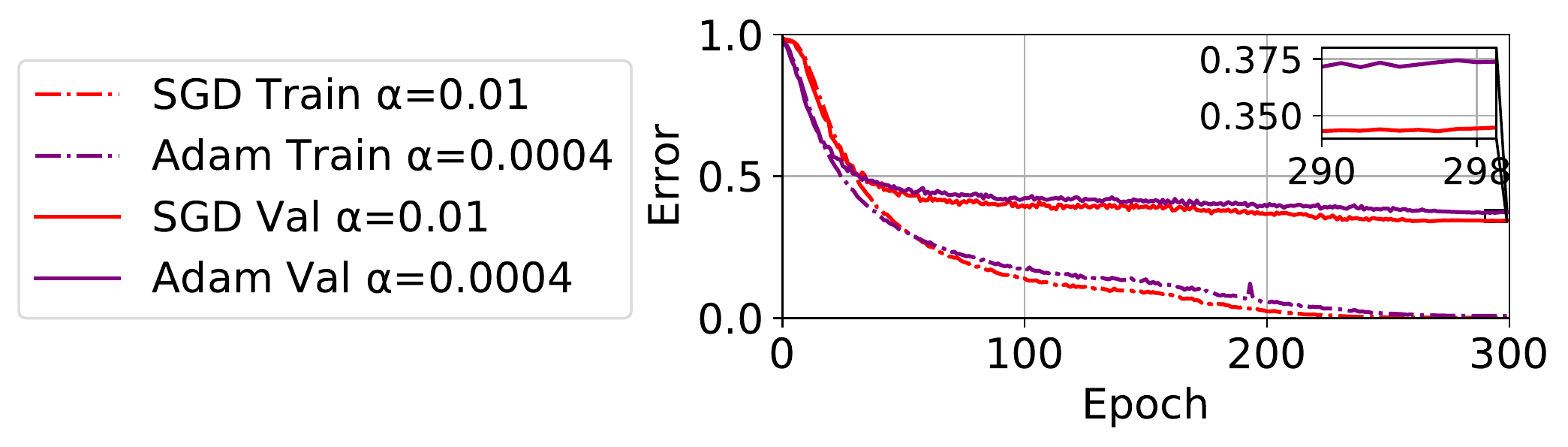}
		\vspace{-15pt}
		\caption{SGD quickly outgeneralises Adam}
		\label{subfig:adamvssgd}
	\end{subfigure}
	\begin{subfigure}[b]{0.49\linewidth}
		\includegraphics[trim=0cm 0 0 0,clip, width=\textwidth]{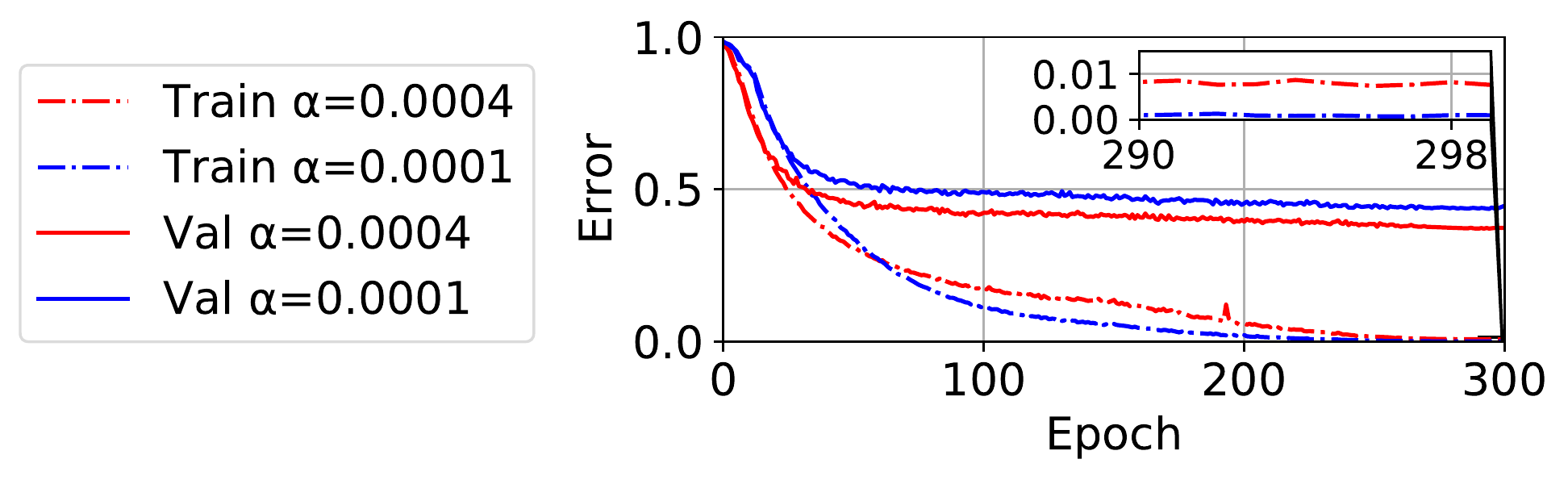}
		\vspace{-15pt}
		\caption{Adam Train/Val for Learning Rates $\{\alpha_{i}\}$}
		\label{subfig:adamlr}
	\end{subfigure}
	\vspace{0pt}
	\caption{\textbf{Adaptive Generalisation Gap and its extent are clearly visible without regularisation.} Train/Val Error on CIFAR-$100$ using VGG-$16$ \latestEdits{without} batch normalisation and weight decay.}
	\label{fig:adamsgd}
\end{figure}
\subsection{Learning Rates}
 Previous work has investigated the relationship
between generalisation and the ratio of learning rates and batch size during SGD optimisation~\cite{jastrzebski2020the}, in addition to
stability analysis \cite{wu2017towards} showing that larger learning rates lead to lower spectral norms (which Bayesian and minimum description length principles suggest lead to better generalisation~\cite{hochreiter1997flat}). \cite{li2019towards} argue that larger learning rates learn hard to generalise, easy to fit patterns better than their lower learning rate counterparts and that this forms part of the generalisation gap. Whilst we expect increased per iteration weight decay $(1-\alpha\gamma)$ from a larger learning rate $\alpha$ with $L_{2}$ regularisation coefficient $\gamma$ to lead to improved generalisation, this effect remains pertinent even with no regularisation as shown in Fig~\ref{subfig:adamlr}. What further remains unclear for adaptive methods is the importance of the \textit{global learning rate}, given that each parameter has its own \textit{individual learning rate}. Additionally, for a damping/numerical stability coefficient $\delta$, the largest possible individual learning rate is given by $\alpha/\delta$, which for typical setups can be orders of magnitude larger than that of SGD. Hence, it is clear that not only the global learning rate is of importance, but also the relative learning rate in different directions, meriting further study.
\subsection{Regularisation}
\latestEdits{One factor hypothesised to contribute to the Adaptive Generalisation Gap is the non-equivalence}
% A primary focus
% for closing the Adaptive Generalisation Gap has been on the
% in-equivalence
between traditional weight decay and $L_{2}$ regularisation for adaptive \cite{loshchilov2018decoupled} and second order \cite{zhang2018three} methods. However, as shown in Fig~\ref{subfig:adamvssgd}, even when no regularisation is employed, using their respective best settings, the generalisation of SGD strongly outperforms that of Adam. This strongly suggests that SGD \textit{inherently} generalises better than adaptive methods and requires further study. Furthermore a strong understanding as to why weight decay strongly outperforms $L_{2}$ regularisation for Adaptive methods remains elusive and in need of further investigation.
\subsection{Flatness}
\latestEdits{The notion of \textit{flatness} (or correspondingly, \textit{sharpness}) has received considerable attention as a property of training loss minima that is predictive of their generalisation}.
\latestEdits{Although there is no universally accepted definition, \textit{flatness} is
\latestEdits{typically}
% usually
defined through properties of the
\latestEdits{Hessian (the second derivative of the loss)},}
% second derivative of the loss, known as the \textit{Hessian},
such as its spectral norm or trace.
% A key concept allowing for the comparison of different training loss minima using only training data, is \textit{flatness}.
Mathematically when integrating out the product of the maximum likelihood (MLE) solution (given by the final weights) with the prior, the posterior is \textit{shifted} relative to the MLE solution. For \textit{sharp} minima, the difference in loss for a small shift is potentially large, \latestEdits{motivating the study of sharpness in the context of generalisation. Indeed,}
the idea of a \textit{shift} between the training and testing loss surface is prolific in the literature and regularly related to generalisation \cite{he2019asymmetric,izmailov2018averaging,maddox2019simple,yao2018hessian,zhang2018theory,keskar2016large}.

However, the lack of reparameterisation invariance \latestEdits{of the Hessian}~\cite{dinh2017sharp}, has subjected its use for predicting generalisation to criticism \cite{neyshabur2017exploring,tsuzuku2019normalized,rangamani2019scale}. Normalized definitions of flatness, have been introduced \cite{tsuzuku2019normalized,rangamani2019scale} in a PAC-Bayesian framework, although are not widely available to practitioners and require specialist implementations.

Whilst \cite{rangamani2019scale} note that empirically Hessian based sharpness measures correlate with generalisation, we find
that \textit{"flatness"} as defined by the spectral/Frobenius norm of the Hessian of the loss can give strongly misleading results when comparing solutions between adaptive and non-adaptive optimisers. As shown in Figures ~\ref{subfig:sgdc10} and ~\ref{subfig:gadamc10}, it is possible to find better generalising solutions with adaptive optimisers that are nonetheless significantly ``sharper'' than those found by SGD. Furthermore sharpness at a point in weight-space may not be indicative of the flatness of the overall basin of attraction. In order to alleviate this, some authors have considered taking random directions in weightspace \cite{izmailov2018averaging} for large distances, however it is unclear to what extent $2D$ heatmaps are representative in what are typically million dimensional spaces. 
\subsection{Related Work:} To the best of our knowledge, there has been no theoretical work analysing the \emph{adaptive generalisation gap}, with the notable exception of \cite{wilson2017marginal}, who consider the poor generalisation performance of adaptive methods to be inherent (and show this on a simple example). Practical amendments to improve generalisation of adaptive methods have included dynamically switching between Adam and SGD \cite{keskar2017improving}, taking the preconditioning matrix in Adam to the power of $p \in [0,1/2]$ ~\cite{chen2018closing}, employing weight decay instead of $L_2$ regularisation \cite{zhang2018three,loshchilov2018decoupled} and altering the damping or numerical stability constant, typically taken as $10^{-8}$ in Adam \cite{choi2019empirical}. While empirically effective, these alterations lack a clear theoretical motivation. None of the aforementioned works explain why Adam (or adaptive methods in general) inherently generalise worse than SGD. Instead, they show that switching from Adam to SGD, or making Adam more similar to SGD brings improvements. A clear analysis of how SGD differs from adaptive methods and why this impacts generalisation is not provided in the literature and this forms the basis of our paper.

\subsection{Contributions}\label{sec:contrib}
We conjecture that a key driver of the adaptive generalisation gap is the fact that adaptive methods \emph{fail to account for the greater levels of noise associated with their estimates of flat directions in the loss landscape}. The fundamental principle underpinning this conjecture---that sharp directions contain information from the underlying process and that flat directions are largely dominated by noise---is theoretically motivated from the spiked covariance model~\cite{baik2004eigenvalues}. This model has been successfully applied in Principal Component Analysis (PCA), covariance matrix estimation and finance \cite{bloemendal2016principal, everson2000inferring,bun2017cleaning,bun2016my}.  We revisit this idea in the context of deep neural network optimisation.

In particular, we consider a spiked additive signal-plus-noise random matrix model for the batch Hessian of deep neural network loss surfaces. In this model, results from random matrix theory suggest several practical implications for adaptive optimisation. We use linear shrinkage theory \cite{bun2016my,bun2016rotational,bun2017cleaning} to illuminate the role of damping in adaptive optimisers and use our insights to construct an adaptive damping scheme that greatly accelerates optimisation. We further demonstrate that typical hyper-parameter settings for adaptive methods produce a systematic bias in favour flat directions in the loss landscape and that the adaptive generalisation gap can be closed by redressing the balance in favour of sharp directions. To track to bias towards flat vs sharp directions we define the \textit{\reflong}:
\begin{equation}
    \reff := \frac{\alpha_{\text{flat}}}{\alpha_{\text{sharp}}}
\end{equation}
where $\alpha_{\mathrm{flat}}$ and $\alpha_{\mathrm{sharp}}$ are the learning rates along the flat and sharp directions,  respectively and this ratio encapsulates the noise-to-signal ratio as motivated by our conjecture (the terms \textit{flat} and \textit{sharp} are defined more precisely in Section \ref{sec:spiked_model}). 
\section{The Spiked Model for the Hessian of the Loss}\label{sec:spiked_model}
\label{sec:theory}

\subsection{Key Result: Sharp directions from the True Loss surface survive, others wash out}

We can rewrite the (random) batch hessian $\mH_{\text{batch}}$ as the combination  of the (deterministic) true hessian $\mH_{\text{true}}$ plus some fluctuations matrix:
\begin{equation}\label{eq:additive_noise}
    \mH_{\text{batch}}(\vw) = \mH_{\text{true}}(\vw) + \mX(\vw).
\end{equation}
In \cite{granziol2020learning} the authors consider the difference between the batch and empirical Hessian, although this is not of interest for generalisation, the framework can be extended to consider the true Hessian. The authors further show, under the assumptions of Lipschitz loss continuity, almost everywhere double differentiable loss and that the data are drawn i.i.d from the data generating distribution that the elements of $\mX(\vw)$ converge to normal random variables\footnote{Note that although a given batch Hessian is a fixed deterministic property, we are interested in generic properties of batches drawn at random from the data generating distribution for which we make statements and can hence model the fluctuations matrix as a random matrix.}. Under the assumptions of limited dependence between and limited variation in the variance of the elements of the fluctuations matrix, the spectrum of the fluctuations matrix converges to the Wigner semi-circle law \cite{granziol2020learning,wigner1993characteristic}, i.e. weakly almost surely
\begin{align}
    \frac{1}{P}\sum_{i=1}^P \delta_{\lambda_i(\mX)} \rightarrow \mu_{SC},
\end{align}
where the $\lambda_i(\mX)$ are the eigenvalues of $\mX$ and $d\mu_{SC}(x) \propto \sqrt{2P^2 - x^2}dx$.
The key intuition in this paper is that sharp directions of the true loss surfaces, that is directions in which the true Hessian has its largest eigenvalues, are more reliably estimated by the batch loss than are the flat directions (those with small Hessian eigenvalues).
This intuition is natural in random matrix theory and is supported by results such as the following.
\nick{\begin{theorem}
	\label{theorem:overlap}
Let $\{\vtheta_i\}_{i=1}^P$,  $\{\vphi\}_{i=1}^P$ be the orthonormal eigenbasis of the true Hessian $\nabla^2 L_{\mathrm{true}}$ and batch Hessian $\nabla^2 L_{\mathrm{batch}}$ respectively. Let also $\nu \geq \ldots \geq \nu_P$ be the eigenvalues of $\nabla^2 L_{\mathrm{true}}$. Assume that $\nu_i = 0$ for all $i > r$, for some fixed $r$. Assume that $\mX$ is a generalised Wigner matrix. Then as $P\rightarrow\infty$ the following limit holds almost surely	\begin{equation}\label{eq:overlap}
		|\vtheta_{i}^{T}\vphi_{i}|^{2} \rightarrow \begin{cases} 1-\frac{P\sigma^{2}}{B\nu{i}^{2}} &\mbox{if } |\nu_{i}| > \sqrt{\frac{P}{B}}\sigma,\\
			0 & \mbox{otherwise}, \end{cases}
	\end{equation}
where $\sigma$ is the sampling noise per Hessian element.
\end{theorem}}

\begin{proof}
    This is a direct application of a result of \cite{capitaine2016spectrum} which is given more explicitly in the case of GOE Wigner matrices by \cite{benaych2011eigenvalues}. In particular, we use a scaling of $\mX$ such that the right edge of the support of its spectral semi-circle is roughly at $P^{1/2}B^{-1/2}\sigma$. The expression in Section 3.1 of \cite{benaych2011eigenvalues} can then be applied to $P^{-1/2}\mH_{\text{batch}}$ and re-scaled in $\sqrt{P}$ to give the result. Note that the substantiation of the expression from \cite{benaych2011eigenvalues} in the case of quite general Wigner matrices is given by Theorem 16 of \cite{capitaine2016spectrum}. 
\end{proof}

% \begin{proof}
% Under the above assumptions set out in \cite{granziol2020learning} (and repeated in the supplementary), The problem reduces to finding the finite rank perturbation of large random matrices for which the result for the Wigner is known \cite{benaych2011eigenvalues}. The result follows by using the scaling relations due to the matrix dimension $P$ and the batch size $B$. We further document details of the assumptions and proof in the Supplementary.
% \end{proof}
\vspace{-5pt}

Results like Theorem \ref{theorem:overlap} are available for matrix models other than Wigner, such as rotationally invariant model \cite{belinschi2017outliers}, and are conjectured to hold for quite general\footnote{Roughly speaking, models for which a local law can be established \cite{erdHos2017dynamical}.} models \cite{benaych2011eigenvalues}. We prove in the appendix section \ref{sec:app_sc}, under some conditions, convergence of the spectral measure of $P^{-1/2}\mX$ to the semi-circle. This is necessary to obtain (\ref{eq:overlap}), but not sufficient. The technicalities to rigorously prove Theorem \ref{theorem:overlap} without assuming a Wigner matrix for $\mX$ are out of scope for the present work, requiring as they would something like an optimal local semi-circle law for $\mX$ \cite{erdHos2017dynamical}. We require only the general heuristic principle from random matrix theory encoded in (\ref{eq:overlap}), namely that \emph{only sharp directions retain information from the true loss surface}. It is expected that this principle will hold for a much wider class of random matrices than those for which it has been rigorously proven. This is acutely important for adaptive methods which rely on curvature estimation, either explicitly for stochastic second order methods or implicitly for adaptive gradient methods.
 %to accelerate their progress on the true loss surface. 

% We highlight a particular insight from the spiked covariance literature \cite{johnstone2001distribution} which implies that \textit{sharper directions are estimated more accurately than flatter directions}. In \cite{granziol2020learning}, the authors assume that the Batch Hessian is given as the True Hessian plus an additive fluctuations matrix. Under this model and the conditions stated in the suppl. material (with full proof in \cite{benaych2011eigenvalues})\nickcomment{I would have an actual proof of this theorem in the appendix, even if it is just a one liner saying ``apply \cite{benaych2011eigenvalues} in this context.''}, namely that the fluctuations or noise matrix must have a bounded per element variance, the differences in element variances cannot be too large and the number of fully dependent rows must be small, we have:
% \nickcomment{If the conditions are significant, we should have some discussion of their plausibility. If they are just technical and believed to be totally reasonable, just relegate them all to the appendix.}

\begin{figure}[h!]
	\centering
	\begin{subfigure}[b]{0.36\linewidth}
		\includegraphics[trim=0cm 0 0 0,clip,width=\textwidth]{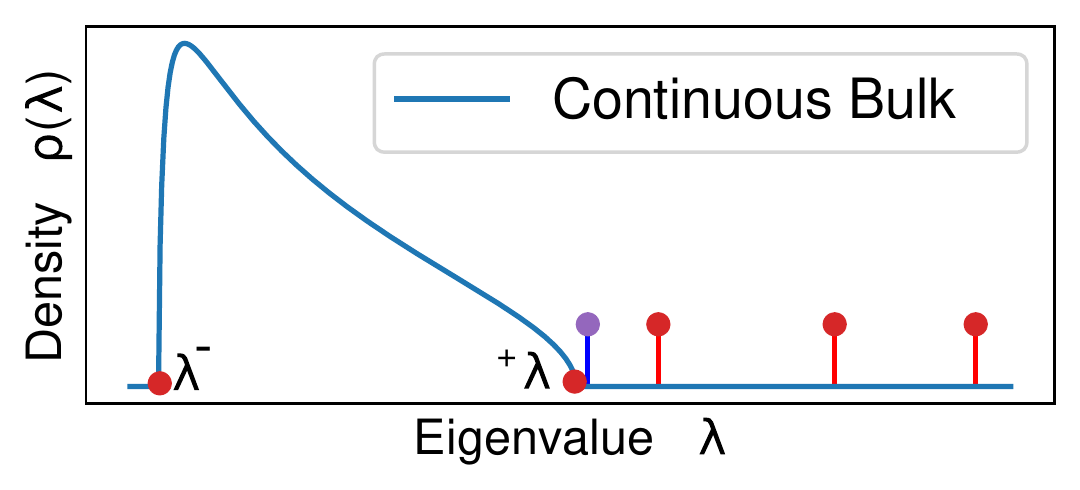}
		\vspace{-15pt}
		\caption{Hypothetical $\rho(\lambda)$}
		\label{subfig:standardmpoutlier}
	\end{subfigure}
	\begin{subfigure}[b]{0.31\linewidth}
		\includegraphics[trim=0cm 0 0 0,clip,width=\textwidth]{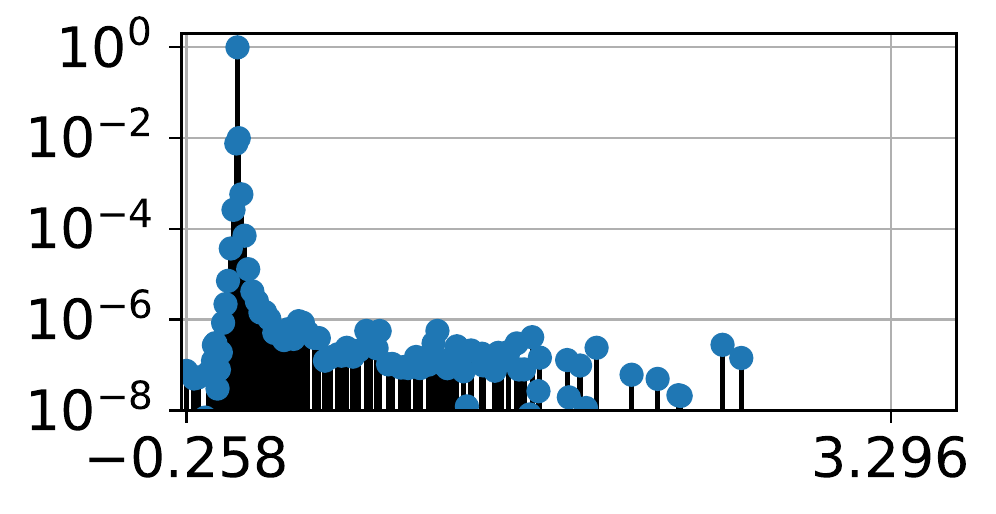}
		\vspace{-15pt}
		\caption{Val Acc$ = 94.3$, SGD}
		\label{subfig:sgdc10}
	\end{subfigure}
	\begin{subfigure}[b]{0.31\linewidth}
		\includegraphics[trim=0cm 0 0 0,clip, width=\textwidth]{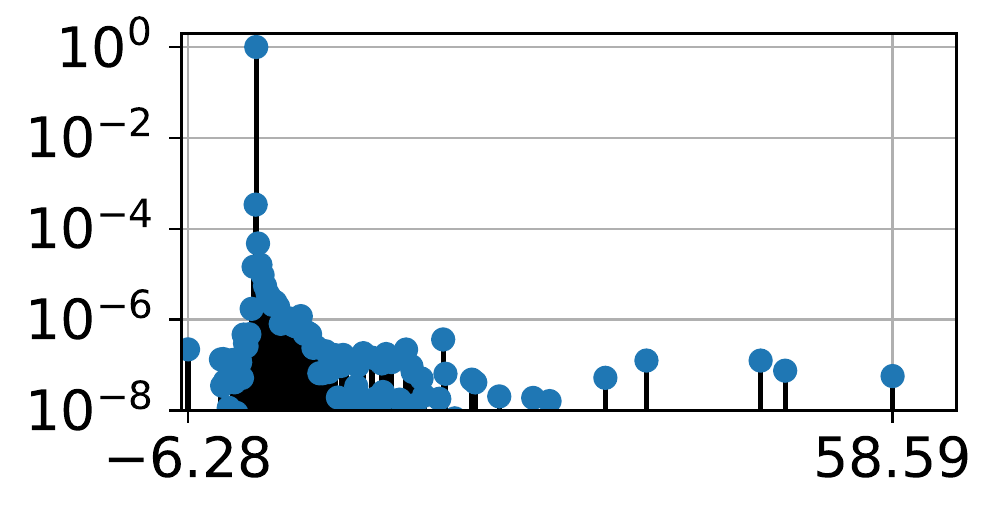}
		\vspace{-15pt}
		\caption{Val Acc$= 95.1$, Adam}
		\label{subfig:gadamc10}
	\end{subfigure}
	% 	\begin{subfigure}[b]{0.5\linewidth}
	% 		\includegraphics[trim=0cm 0 0 0,clip, width=\textwidth]{figs/AIstats_VGG16_Hessian_Train_Epoch=150_Network=VGG16_dataset=CIFAR100.png}
	% 		\vspace{-15pt}
	% 		\caption{VGG-$16$ C-$100$ Hessian}
	% 		\label{subfig:vgghessian}
	% 	\end{subfigure}
	\caption{(a) Hypothetical spectral density plot with a sharply supported continuous bulk region, a finite size fluctuation {\color{blue}shown in blue} corresponding to the Tracy-Widom region and three well-separated outliers {\color{red}shown in red}. (b,c) VGG-$16$ Hessian on the CIFAR-$10$ dataset at epoch $300$ for SGD and Adam respectively. Note the "sharper" solution has better validation accuracy.}
	\label{tab:whatissharp}
	\vspace{-10pt}
\end{figure}

% \subsection{Intuition: why \textit{sharp} directions generally retain more information than \textit{flat} directions} \label{subsec:intuition}
\medskip
The spectrum of the noise matrix occupies a continuous region that is sharp in the asymptotic limit \cite{bun2017cleaning} known as \textit{bulk} supported between $[\lambda_{-},\lambda_{+}]$ \cite{bun2017cleaning,bun2016my,bun2016rotational} and observed in DNNs  \cite{granziol2019mlrg,papyan2018full,sagun2017empirical}. Within this bulk
\latestEdits{eigenvectors are uniformly distributed on the unit sphere \cite{benaych2011eigenvalues} and 
all information about the original eigenvalue/eigenvector pairs is lost \cite{baik2005phase}.} Hence from a theoretical perspective it makes no sense to estimate these directions and move along them accordingly. 
% \latestEdits{ and these}
% . These
An eigenvalue, $\lambda_i$, corresponds to a \textit{flat} direction if $\lambda_{i} \leq \lambda_{+}$. For finite-size samples and network size, there exists a region beyond the predicted asymptotic support of the noise matrix, called the Tracy--Widom region \cite{tracy1994level, el2007tracy}, where there may be isolated eigenvalues which are part of the noise matrix spectrum (also shown in Fig.~\ref{subfig:standardmpoutlier}). \nick{The width of the Tracy--Widom region is very much less than that of the bulk.} Anything beyond the Tracy--Widom region $\lambda_{i} \gg \lambda_{+}$, $\lambda_{i} \ll \lambda_{-}$ is considered an outlier and corresponds to a \textit{sharp} direction. \emph{Such directions represent underlying structure from the data}. The eigenvectors corresponding to these eigenvalues can be shown to lie in a cone around their true values \cite{benaych2011eigenvalues} (see Theorem \ref{theorem:overlap}).
In Fig.~\ref{subfig:sgdc10}, we show the Hessian of a VGG-$16$ network at the $300\textsuperscript{th}$ epoch on CIFAR-100. Here, similar to our hypothetical example, we see a continuous region, followed by a number of eigenvalues which are close to (but not within) the bulk, and finally, several clear outliers.

\section{Detailed experimental investigation of Hessian directions}\label{sec:logisticexp}
\label{subsec:mnist}
In this section we seek to validate our conjecture that movements in the sharp direction of the loss landscape are inherently vital to generalisation by studying a convex non-stochastic example. For such a landscape there is only a single global minimum and hence discussions of bad minima are not pertinent.
We implement a second-order optimiser based on the Lanczos iterative algorithm \cite{meurant2006lanczos} (LanczosOPT) against a gradient descent (GD) baseline. We employ a training set of $1$K MNIST \cite{lecun1998mnist} examples using logistic regression and validate on a held out test set of $10$K examples. Each optimiser is run for $500$ epochs. The Lanczos algorithm
 is an iterative algorithm for learning a subset of the eigenvalues/eigenvectors of any Hermitian matrix, requiring only matrix--vector products. When the number of Lanczos steps, $m$, is significantly larger than the number of outliers, the outliers in the spectrum are estimated effectively~\cite{granziol2019mlrg}. Since the number of well-separated outliers from the spectral bulk is at most the number of classes \cite{papyan2018full} (which is $n_{c}=10$ for this dataset), we expect the Lanczos algorithm to pick out these well-separated outliers when the number of iterations $k \gg n_{c}$ \cite{granziol2019mlrg,meurant2006lanczos} and therefore use $k=50$. 
To investigate the impact of scaling steps in the Krylov subspace given by the sharpest directions, we consider the update $\vw_{k+1}$ of the form:
\begin{equation}
\label{eq:lanczosopt}
\vw_{k} -\alpha\bigg(\frac{1}{\eta}\sum_{i=1}^{k}\frac{1}{\lambda_{i}+\delta}\vphi_{i}\vphi_{i}^{T}\nabla L(\vw_{k})+\sum_{i=k+1}^{P}\frac{1}{\delta}\vphi_{i}\vphi_{i}^{T}\nabla L(\vw_{k})\bigg) 
\end{equation}
where $P=7850$ (the number of model parameters) and hence the vast majority of flat directions remain unperturbed. To explore the effect of sharp directions more explicitly, we have introduced perturbations to the optimiser (denoted LOPT$[\eta]$), in which we reduce the first term in the parenthesis of Equation~\ref{eq:lanczosopt} by a factor of $\eta$ (we explore scaling factors of $3$ and $10$). This reduces movement in sharp directions, consequently increases reliance on flat directions (which are left largely unperturbed). 
For \latestEdits{a} fixed $\alpha$, $\delta$ controls the \reflong. 
\label{subsec:logisticexp}
\vspace{-5pt}
\paragraph{Experimental Results:} 
In Fig.~\ref{fig:logisticheatmap}, where we show in heat map form the difference from the best training and testing error as a function of $\delta$ and $\eta$, we observe evidence consistent with our central hypothesis. As we increase $\reff$ (by decreasing the value of $\delta$ for a fixed $\alpha$ value of $0.01$), the generalisation of the model suffers correspondingly. 

\begin{figure}
	\begin{subfigure}[b]{0.48\linewidth}
		\includegraphics[width=\textwidth]{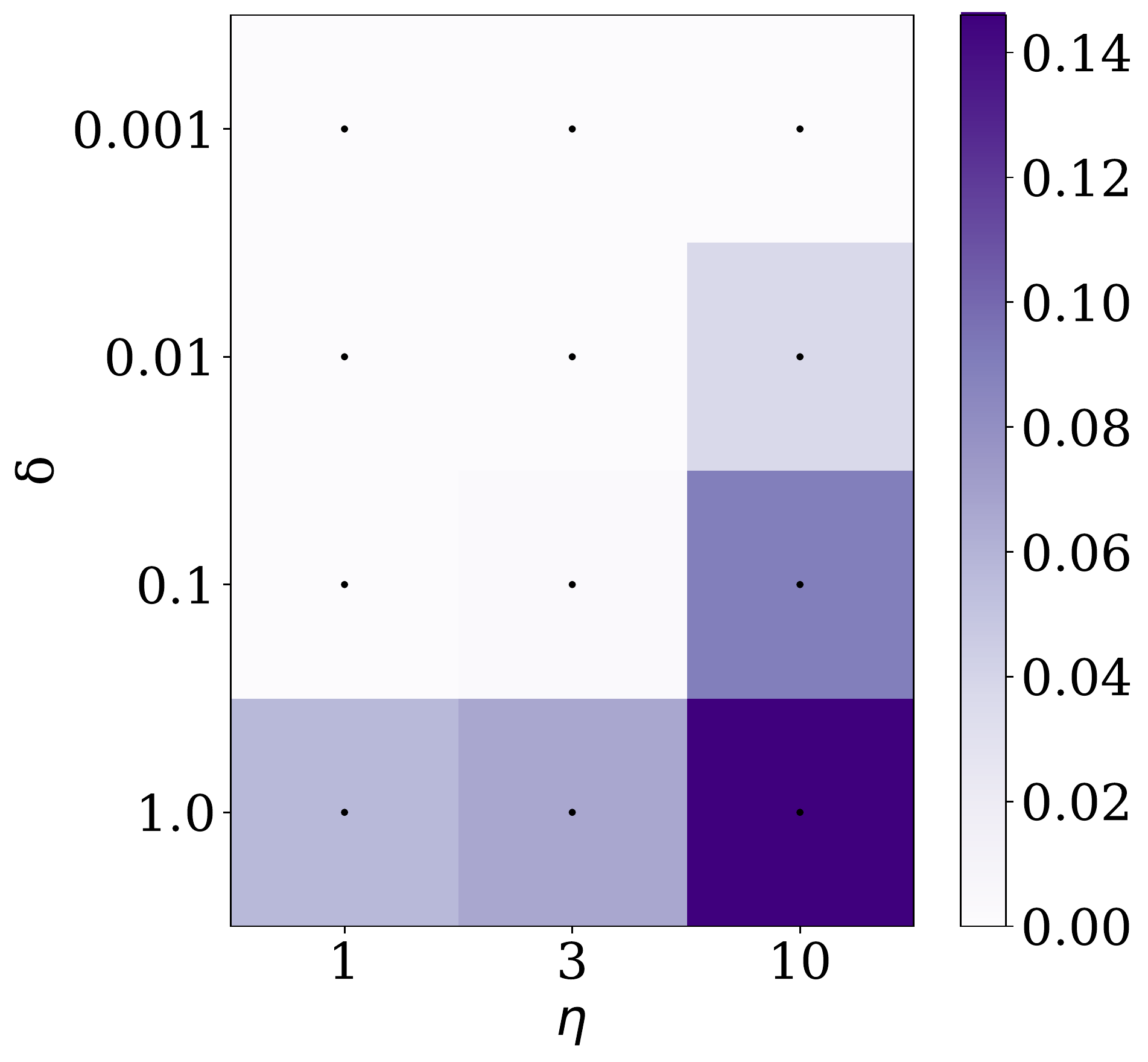}
		\caption{$\Delta(\delta,\eta)$ Training}
	\end{subfigure}
	\begin{subfigure}[b]{0.48\linewidth}
		\includegraphics[width=\textwidth]{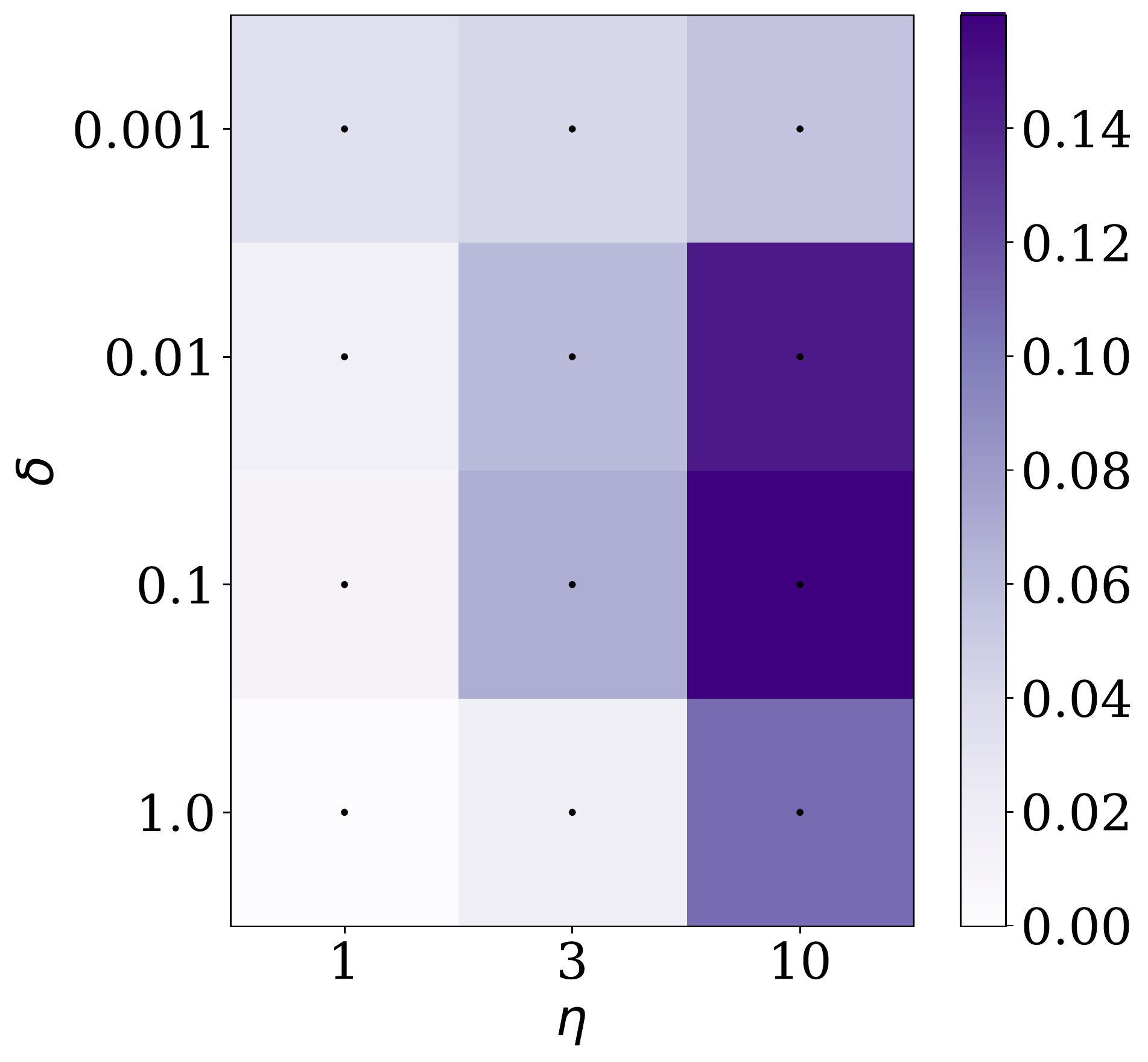}
		\caption{$\Delta(\delta,\eta)$  Testing}
	\end{subfigure}
	% 		\vspace{-0.3cm}
	\captionof{figure}{Error change with damping/sharp direction perturbation $\delta, \eta$ in LanczosOPT, relative to the single best run. Darker regions indicate higher error.}
	\label{fig:logisticheatmap}
	% 		\vspace{-10pt}
\end{figure}

\begin{figure}[h!]
	\centering
	\begin{subfigure}[b]{0.57\textwidth}
		\includegraphics[width=\textwidth,trim={0cm 0.2cm 0 0},clip]{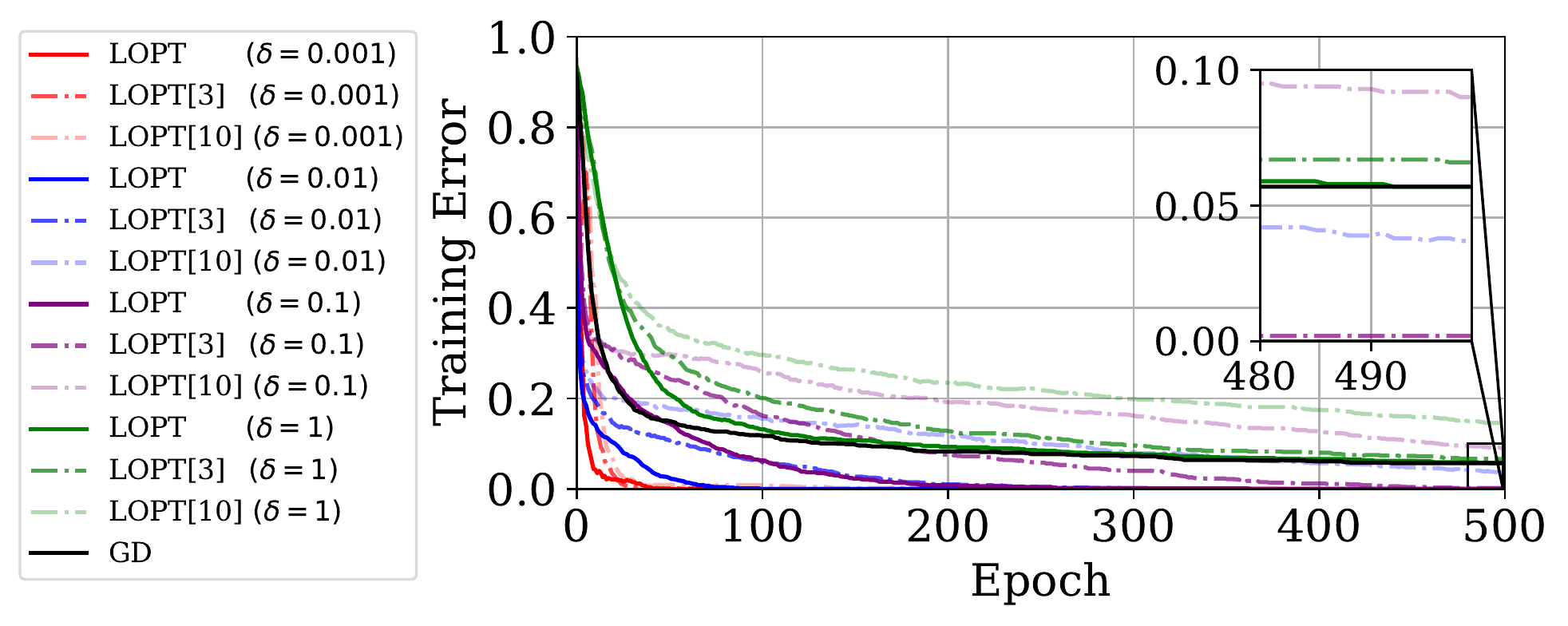}
	\end{subfigure}
	\begin{subfigure}[b]{0.42\textwidth}
		\includegraphics[width=\textwidth,trim={0cm 0.25cm 0 0},clip]{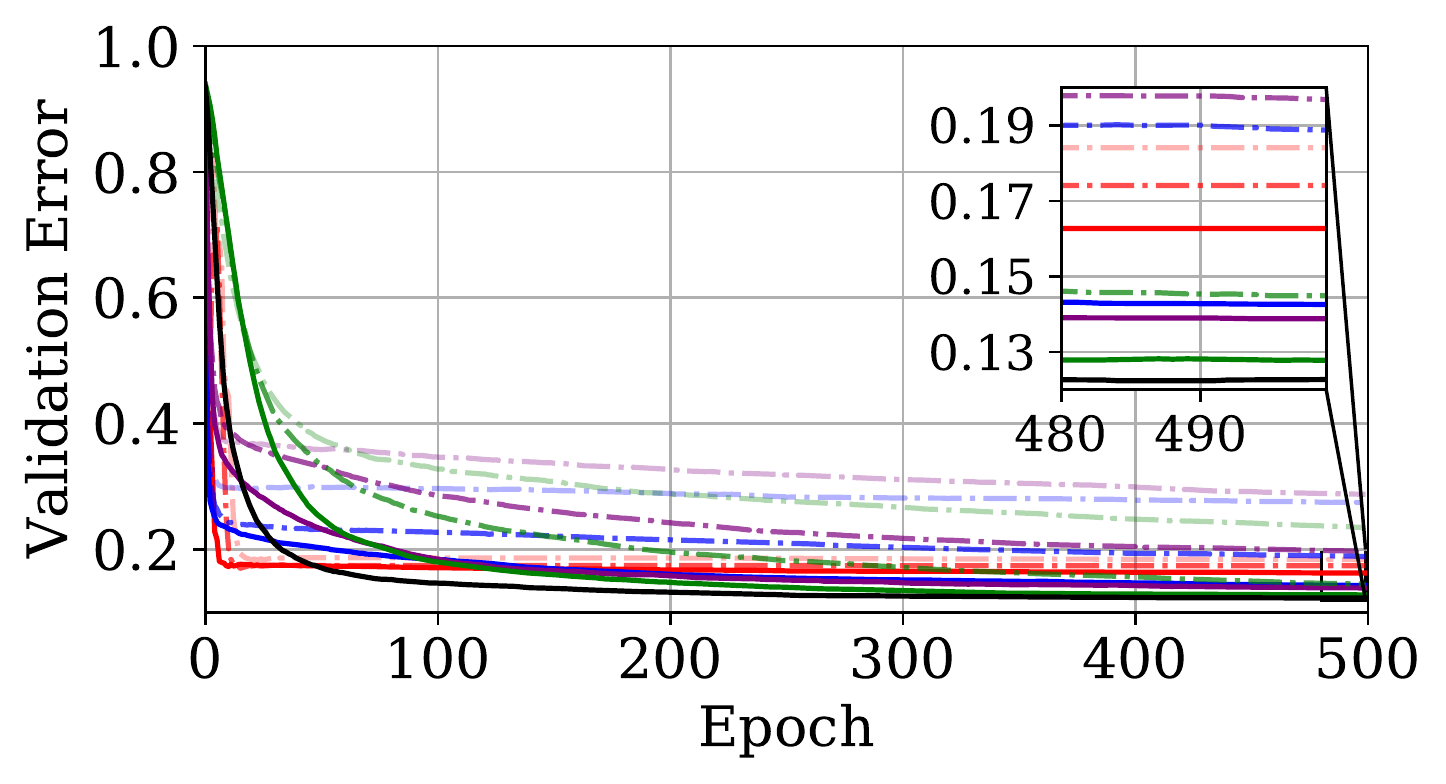}
	\end{subfigure}
	\vspace{-0.3cm}
	\caption{Training/test error of LanczosOPT/Gradient Descent (LOPT/GD) optimisers for logistic regression on the MNIST dataset with fixed learning rate $\alpha=0.01$ across different damping values, $\delta$. LOPT$[\eta]$ denotes a modification to the LOPT algorithm that perturbs a subset of update directions by a factor of $\eta$. Best viewed in colour. }
	\label{fig:logisticlrandeps}
\end{figure}

For each fixed value of $\delta$, we see clearly that perturbations of greater magnitude cause greater harm to generalisation than training. We also note that for larger values of $\delta$ the perturbed optimisers suffer more gravely in terms of the effect on both training and validation.
We show the full training curves in Figure \ref{fig:logisticlrandeps}.
We observe that the generalisation of all algorithms is worsened by explicit limitation of movement in the sharp directions (and an increase of \reflong), however for extremely low damping measures (which are typical in adaptive optimiser settings) there is no or very minimal impact in training performance (upper region of Fig.~\ref{fig:logisticheatmap} (a)).

\paragraph{Fashion MNIST:}
% Given the notoriety of MNIST working for all algorithms and theories
\latestEdits{We repeat} the experimental procedure for the FashionMNIST dataset\latestEdits{~\cite{xiao2017fashion}}, which paints an identical picture (at slightly higher testing error) The full training curves are given in Figure \ref{fig:fashionlogisticlrandeps}.
\begin{figure}[h!]
	\centering
	\begin{subfigure}[b]{0.57\textwidth}
		\includegraphics[width=\textwidth,trim={0cm 0.2cm 0 0},clip]{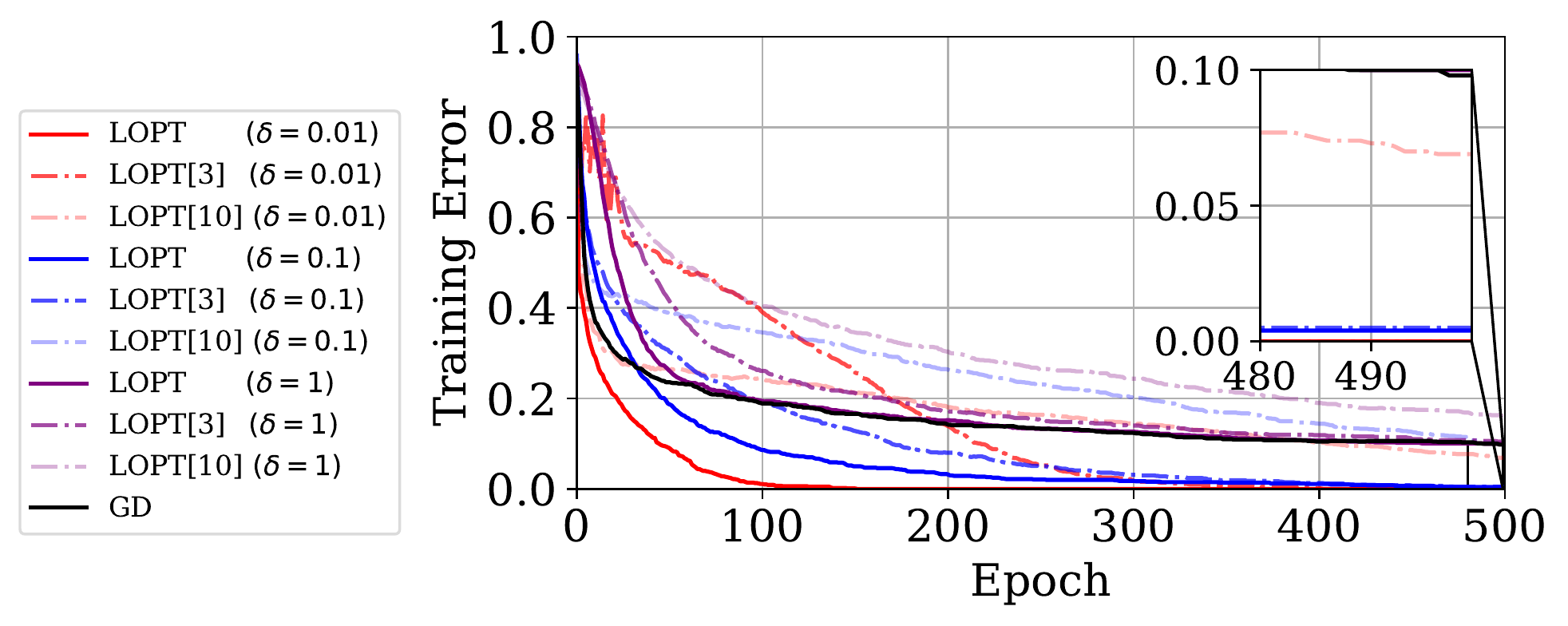}
	\end{subfigure}
	\begin{subfigure}[b]{0.42\textwidth}
		\includegraphics[width=\textwidth,trim={0cm 0.25cm 0 0},clip]{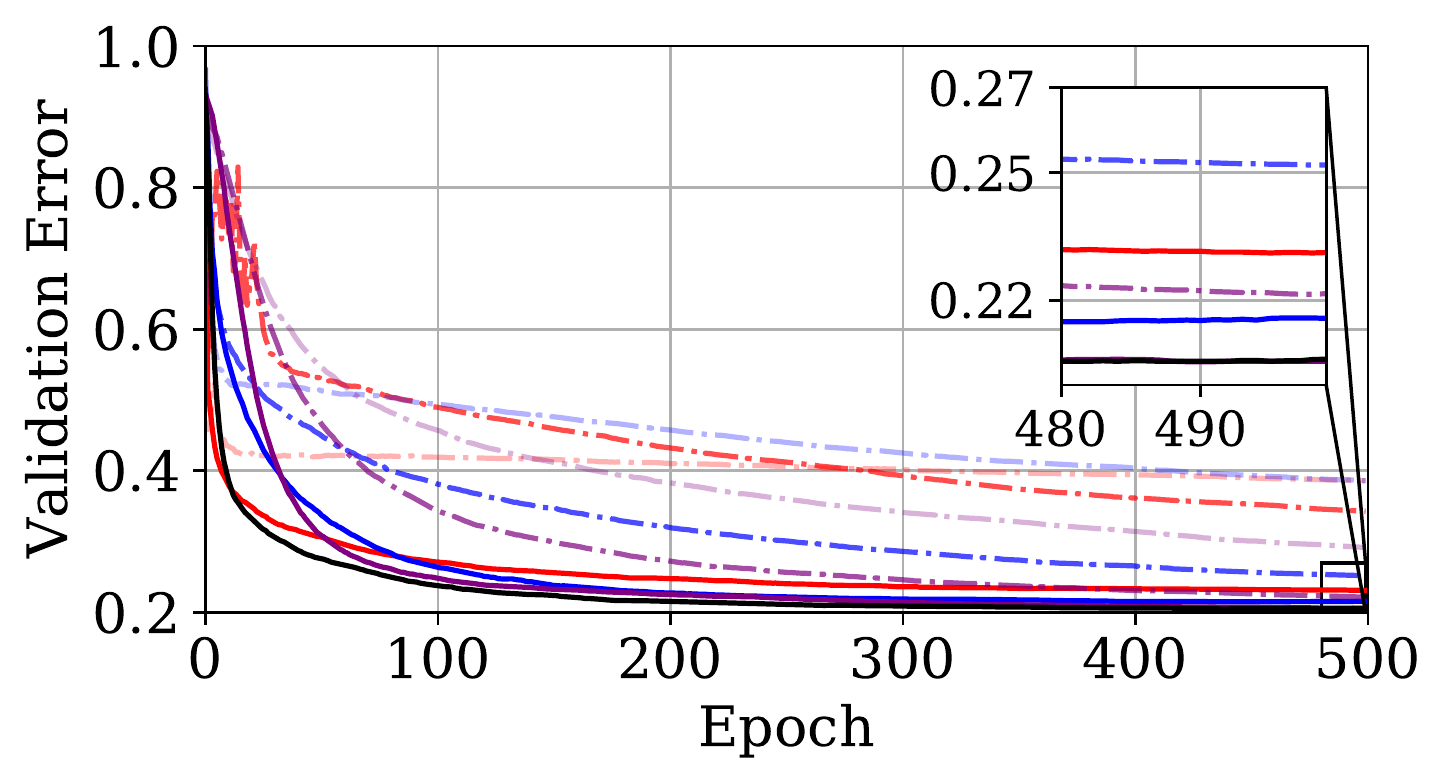}
	\end{subfigure}
	\vspace{-0.3cm}
	\caption{Training/test error of LanczosOPT/Gradient Descent (LOPT/GD) optimisers for logistic regression on the FashionMNIST dataset with fixed learning rate $\alpha=0.01$ across different damping values, $\delta$. LOPT$[\eta]$ denotes a modification to the LOPT algorithm that perturbs a subset of update directions by a factor of $\eta$. Best viewed in colour. }
	\label{fig:fashionlogisticlrandeps}
\end{figure}
% Fig~\ref{fig:fashionlogisticlrandeps} of the Supp. Matt.

\section{The role of damping}\label{sec:nnexperiments}
Consider a general iterative optimiser that seeks to minimise the scalar loss $L(\vw)$ for a set of model parameters $\vw \in \mathbb{R}^P$. Recall the $k+1$-th iteration of such an optimiser can be written\footnote{Ignoring additional features such as momentum and explicit regularisations.} as follows:
\begin{equation}
\vw_{k+1} \leftarrow \vw_{k} - \alpha_k \mB^{-1} \nabla L_{\mathrm{batch}}(\vw_{k})
\end{equation}
where $\alpha_k$ is the global learning rate. For SGD, $\mB = \mI$ whereas for adaptive methods, $\mB$ typically
% provides
\latestEdits{comprises}
some form of approximation to the Hessian i.e. $\mB \approx \nabla^{2}L_{\mathrm{batch}}(\vw_{k})$. 
Writing this update in the eigenbasis of the Hessian\footnote{We assume this to be positive definite or that we are working with a positive definite approximation thereof.} $\nabla^{2}L_{\mathrm{batch}}(\vw_{k}) = \sum_{i}^{P}\lambda_{i}\vphi_{i}\vphi_{i}^{T} \in \mathbb{R}^{P\times P}$, where $\lambda_1\geq \lambda_2\geq \dots \geq \lambda_{P} \geq 0$ represent the ordered scalar eigenvalues, the parameter step takes the form:
\begin{equation}
\label{eq:secondorderopt}
\begin{aligned}
\vw_{k+1} = \vw_{k} - \sum_{i=1}^{P}\frac{\alpha}{\lambda_{i}+\delta}\vphi_{i}\vphi_{i}^{T}\nabla L_{\nick{\mathrm{batch}}}(\vw_{k}).
\end{aligned}
\end{equation}
Here, $\delta$ is a damping (or numerical stability) term. This damping term (which is typically grid searched \cite{dauphin2014identifying} or adapted during training \cite{martens2015optimizing}) can be interpreted as a trust region \cite{dauphin2014identifying} that is required to stop the optimiser moving too far in directions deemed flat ($\lambda_{i} \approx 0$), known to dominate the spectrum in practice \cite{granziol2020learning,papyan2018full,ghorbani2019investigation}, and hence diverging. In the common adaptive optimiser Adam \cite{kingma2014adam}, it is set to $10^{-8}$. For small values of $\delta$, $\alpha$ must also be small to avoid optimisation instability, hence global learning rates and damping are coupled in adaptive optimisers.
\subsection{Adaptive updates, damping and the \reflong} \label{subsec:dampingadaptive}
The learning rate in the flattest ($\lambda \approx 0$) directions is approximately $\frac{\alpha}{\delta}$, which is larger than the learning rate in the sharpest ($\lambda_{i} \gg \delta$) directions  $\frac{\alpha}{\delta+\lambda_{i}}$. 
This difference in per direction effective learning rate makes the best possible (damped) training loss reduction under the assumption that the loss function can be effectively modelled by a quadratic \cite{Martens2016}. Crucially, however, it is agnostic to how accurately each eigenvector component of the update estimates the true underlying loss surface, which is described in Theorem \ref{theorem:overlap}. Assuming that the smallest eigenvalue $\lambda_{P} \ll \delta$,  we see that $\reff = 1+ \frac{\lambda_{1}-\lambda_{P}}{\delta}$. This is in contrast to SGD where 
$\vw_{k+1} = \vw_{k} - \sum_{i=1}^{P}\alpha\vphi_{i}\vphi_{i}^{T}\nabla L_{\nick{\mathrm{batch}}}(\vw_{k})$ and hence $\reff = 1$. Note that we can ignore the effect of the overlap between the gradient and the eigenvectors of the batch Hessian because we can rewrite the SGD update in the basis of the batch Hessian eigenvectors and hence reduce the problem to one of the relative learning rates.

The crucial point to note here is that the difference in $\reff$ is primarily controlled by the damping parameter: smaller values yield a larger $\reff$, skewing the parameter updates towards flatter directions.

% From Theorem \ref{theorem:overlap} we see that overlap is reduced proportionally to the number of network parameters $P$ and inversely to the number of samples taken for evaluation $B$, and so we expect the influence of this effect to become more significant for larger models.

To further explore our central conjecture for modern deep learning architectures (where a large number of matrix--vector products is infeasible) we employ the KFAC \cite{martens2015optimizing} and Adam \cite{kingma2014adam} optimisers on the VGG-$16$~\cite{simonyan2014very} network on the CIFAR-$100$~\cite{krizhevsky2009learning} dataset. The VGG-16 allows us to isolate the effect of $\reff$, as opposed to the effect of different regularisation implementations for adaptive and non-adaptive methods as discussed by \cite{loshchilov2018decoupled,zhang2018three}.
%  We run all CIFAR experiments on a single Nvidia GeForce RTX 2080 and for ImageNet utilise a cluster of $8$.
\subsection{VGG16: a laboratory for adaptive optimisation}
The deep learning literature contains very many architectural variants of deep neural networks and a large number of engineering ``tricks'' which are employed to obtain state of the art results on a great variety of different tasks. The theory supporting the efficacy of such tricks and architectural designs is often wanting and sometimes entirely absent. Our primary objective in this work is to illuminate some theoretical aspects of adaptive optimisers such as appropriate damping and Hessian estimation, so we require a simple and clean experimental environment free from, where possible, interference from as many different competing effects. To this end, the VGG architecture \cite{simonyan2014very} for computer vision is particularly appropriate. With 16 layers, the VGG has over $16$ million parameters and is capable of achieving competitive test error on a variety of standard computer vision datasets while being trained without batch normalisation \cite{ioffe2015batch} or weight decay. Indeed, features such as weight decay and batch normalisation obscure the effect of learning rate and damping, meaning that even quite poor choices can ultimately give reasonable results given sufficient training iterations\cite{granziol2020learning}. In contrast the VGG clearly exposes the effects of learning rate and damping, with training being liable to fail completely or diverge if inappropriate values are used. Furthermore as shown in \cite{granziol2020learning} the VGG is highly unstable if too large a learning rate is used. This allows us to very explicitly test whether amendments provided by theory are helpful in certain contexts, such as training stability, as unstable training very quickly leads to divergence.

\paragraph{Learning Rate Schedule} For all experiments unless specified,  we use the following learning rate schedule for the learning rate at the $t$-th epoch:
\begin{equation}
	\alpha_t = 
	\begin{cases}
		\alpha_0, & \text{if}\ \frac{t}{T} \leq 0.5 \\
		\alpha_0[1 - \frac{(1 - r)(\frac{t}{T} - 0.5)}{0.4}] & \text{if } 0.5 < \frac{t}{T} \leq 0.9 \\
		\alpha_0r, & \text{otherwise}
	\end{cases}
\end{equation}
where $\alpha_0$ is the initial learning rate. $T$ is the total number of epochs budgeted for all CIFAR experiments. We set $r = 0.01$ for all experiments.

\begin{table*}[ht]
	\vspace{-0.1cm}
	\begin{minipage}[b]{0.98\linewidth}\centering
		\vspace{15pt}
		\hspace{-1.5cm}
		\begin{minipage}[b]{0.40\linewidth}\centering
			\centering
			\setlength\tabcolsep{3.5pt} 
			\begin{tabular}{@{}lrr@{}}
				\toprule
				& \multicolumn{2}{c}{$\alpha$} \\
				\cmidrule(lr){2-3}
				\multicolumn{1}{c}{$\delta$} &  \multicolumn{1}{c}{0.0004} &  \multicolumn{1}{c}{0.001}  \\
				\midrule
				$1$e-$7$ &  \textbf{53.1}(62.9) & \textbf{}  \\
				$4$e-$4$ &  \textbf{21.1}(64.5) & \textbf{}  \\
				$1$e-$3$ & \textbf{9.9}(63.5) & \textbf{20.8}(64.4)  \\
				$5$e-$3$ & \textbf{} & \textbf{9.1}(66.2)  \\
				$8$e-$3$ & \textbf{} & \textbf{2.4}(65.8) \\
				\bottomrule
			\end{tabular}
			\vspace{6pt}
		\end{minipage}
		\hspace{0.1cm}
		\begin{minipage}[b]{0.40\linewidth}
			\centering
			\setlength\tabcolsep{3pt} 
			\begin{tabular}{@{}lrrr@{}}
				\toprule
				& \multicolumn{3}{c}{$\alpha$} \\
				\cmidrule(lr){2-4}
				\multicolumn{1}{c}{$\delta$} & \multicolumn{1}{c}{0.1} & \multicolumn{1}{c}{0.001} & \multicolumn{1}{c}{0.0001} \\
				\midrule
				$1$e-$1$  & \textbf{6.7}(65.0) & \textbf{} & \\
				$1$e-$2$  & & \textbf{20.8}(64.8) & \textbf{} \\
				$1$e-$3$  & & \textbf{} &\textbf{48.2}(62.2) \\
				$3$e-$4$  & & \textbf{} &  \textbf{527.2}(60.2) \\
				$1$e-$4$  & & \textbf{} &  \textbf{711.3}(56.0) \\ \bottomrule
			\end{tabular}
			\vspace{6pt}
		\end{minipage}
		\caption{\textbf{Spectral norms and generalisation}. We report the spectral norm $\lambda_{1}$ at the end of training in \textbf{bold}, with corresponding validation accuracy (in parentheses) for learning rate/damping $\alpha,\delta$ using Adam (left) and KFAC (right) to train a VGG-$16$ network on CIFAR-$100$.}
		\label{tab:tableofspectralnorm}
	\end{minipage}
\end{table*}

\subsection{KFAC with VGG-16 on CIFAR-100:} 
%\begin{figure*}[t!]
%	\centering
%	\begin{subfigure}[b]{0.30\textwidth}
%		\vspace{-0.1cm}
%		\includegraphics[width=\textwidth]{figs/gen_train_kfac_VGG16_CIFAR100.pdf}
%	\end{subfigure}
%	\begin{subfigure}[b]{0.19\textwidth}
%		\includegraphics[width=\textwidth]{figs/gen_test_kfac_VGG16_CIFAR100.pdf}
%	\end{subfigure}
%%	\caption{Training/validation error of the KFAC optimiser for VGG-$16$ on the CIFAR-$100$ dataset with various learning rates $\alpha$ and damping values, $\delta$.}
%%	\label{fig:kfaclrandeps}
%%	\vspace{-14pt}
%	\begin{subfigure}[b]{0.28\textwidth}
%		\includegraphics[width=\textwidth]{figs/Adam_VGG16_delta_train.pdf}
%	\end{subfigure}
%	\begin{subfigure}[b]{0.19\textwidth}
%		\includegraphics[width=\textwidth]{figs/Adam_VGG16_delta_val.pdf}
%	\end{subfigure}
%	\vspace{-5pt}
%	\caption{Training/validation error of the Adam optimiser for VGG-$16$ on the CIFAR-$100$ dataset with various learning rates $\alpha$ and damping values, $\delta$.}
%	\label{fig:adamlrandeps}
%\end{figure*}
%\label{subsec:kfacvgg16}

By decreasing the global learning rate $\alpha$ whilst keeping the damping-to-learning-rate ratio $\kappa = \frac{\delta}{\alpha}$ constant, we increase the \reflong, $\reff$, which is determined by $\frac{\lambda_{i}}{\kappa \alpha}+1$. As shown in Tab.~\ref{tab:tableofspectralnorm} and in Figure \ref{fig:kfaclrandeps} we observe that as we increase $\reff$ the training performance is effectively unchanged, but generalisation suffers ($35\%\rightarrow37.8\%$). 
Whilst decreasing the damping results in poor training for large learning rates, for very low learning rates the network efficiently trains with a lower damping coefficient. Such regimes further increase $\reff$ and we observe that they generalise more poorly. For $\alpha = 0.0001$ dropping the damping coefficient $\delta$ from $0.0003$ to $0.0001$ drops the generalisation further to $60.2\%$ and then $56\%$ respectively. Similar to logistic regression, for both cases the drop in generalisation is significantly larger than the drop in training accuracy. 
\begin{figure*}[h!]
	\centering
	\begin{subfigure}[b]{0.60\textwidth}
	    \vspace{-0.1cm}
		\includegraphics[width=\textwidth]{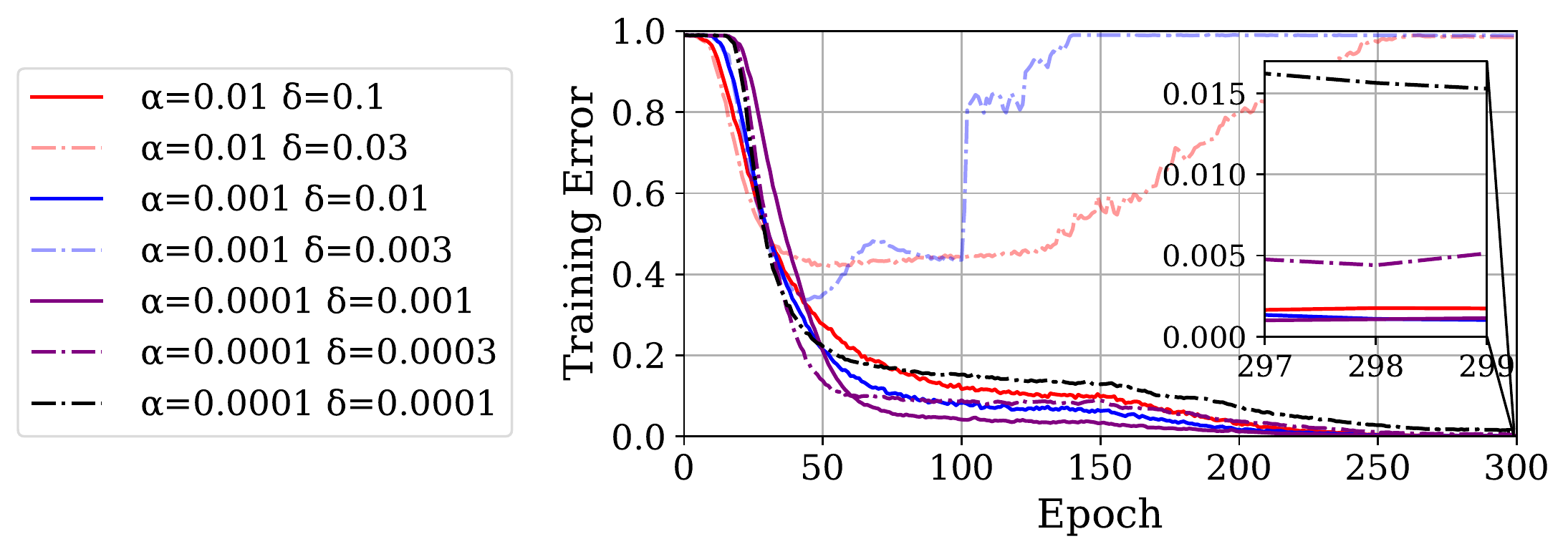}
	\end{subfigure}
	\begin{subfigure}[b]{0.39\textwidth}
		\includegraphics[width=\textwidth]{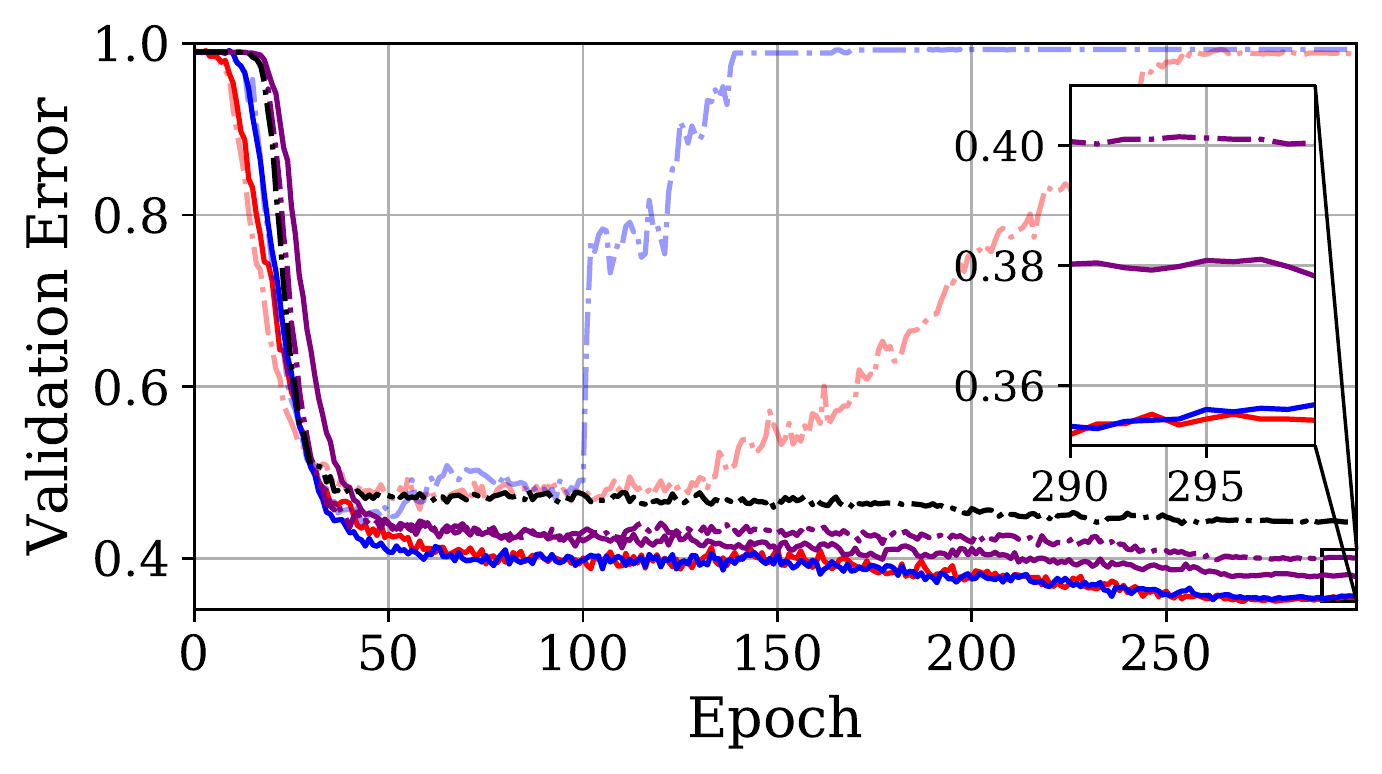}
	\end{subfigure}
	\vspace{-12pt}
	\caption{Training/validation error of the KFAC optimiser for VGG-$16$ on the CIFAR-$100$ dataset with various learning rates $\alpha$ and damping values, $\delta$.}
	\label{fig:kfaclrandeps}
\end{figure*}

\paragraph{Adam with VGG-16 on CIFAR-100:}
\label{subsec:adamexp}
We employ Adam with a variety of learning rate and damping coefficients with results as shown in Tab.~\ref{tab:tableofspectralnorm} and in Figure \ref{fig:adamlrandeps} and compare against a baseline SGD with $\alpha = 0.01$ (corresponding to optimal performance). For the largest learning rate with which Adam trains ($\alpha = 0.0004$) with the standard damping coefficient $\delta = 10^{-8}$, we see that Adam under-performs SGD, but that this gap is reduced by simply increasing the damping coefficient without harming performance. Over-damping decreases the performance.  For larger global learning rates enabled by a significantly larger than default damping parameter, when the damping is set too low, the training is unstable (corresponding to the dotted lines). Nevertheless, many of these curves with poor training out-perform the traditional setting on testing. We find that for larger damping coefficients $\delta = 0.005, 0.0075$ Adam is able to match or even beat the SGD baseline, whilst converging faster. We show that this effect is statistically significant in Tab.~\ref{tab:seeds}. This provides further evidence that for real problems of interest, adaptive methods are not worse than their non-adaptive counterparts as argued by \cite{wilson2017marginal}. We note as shown in Tab.~\ref{tab:tableofspectralnorm}, that whilst increasing $\delta$ always leads to smaller spectral norm, this does not always coincide with better generalisation performance. We extend this experimental setup to include both batch normalisation \cite{ioffe2015batch} and decoupled weight decay \cite{loshchilov2018decoupled}. We use a learning rate of $0.001$ and a decoupled weight decay of $[0,0.25]$. For this experiment  using a larger damping constant slightly assists training and improves generalisation, both with and without weight decay.

\begin{figure*}[t!]
	\centering
	\begin{subfigure}[b]{0.57\textwidth}
		\includegraphics[width=\textwidth]{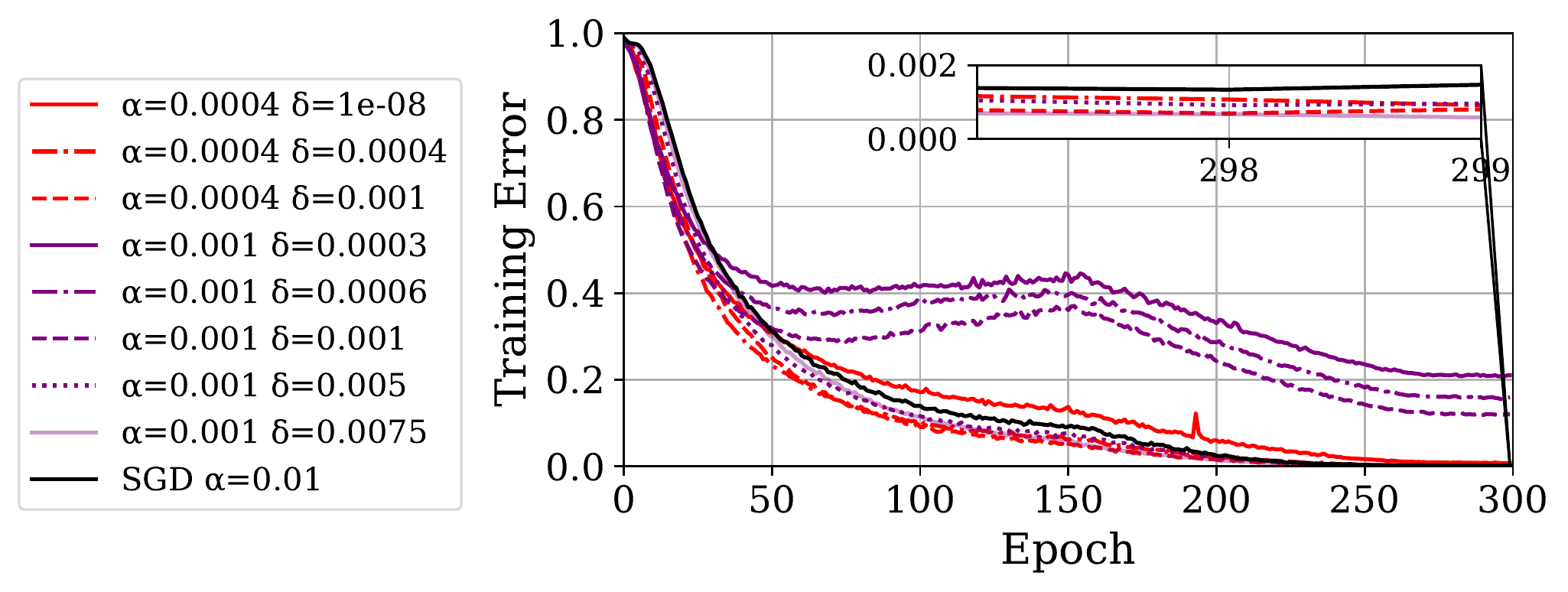}
	\end{subfigure}
	\begin{subfigure}[b]{0.40\textwidth}
		\includegraphics[width=\textwidth]{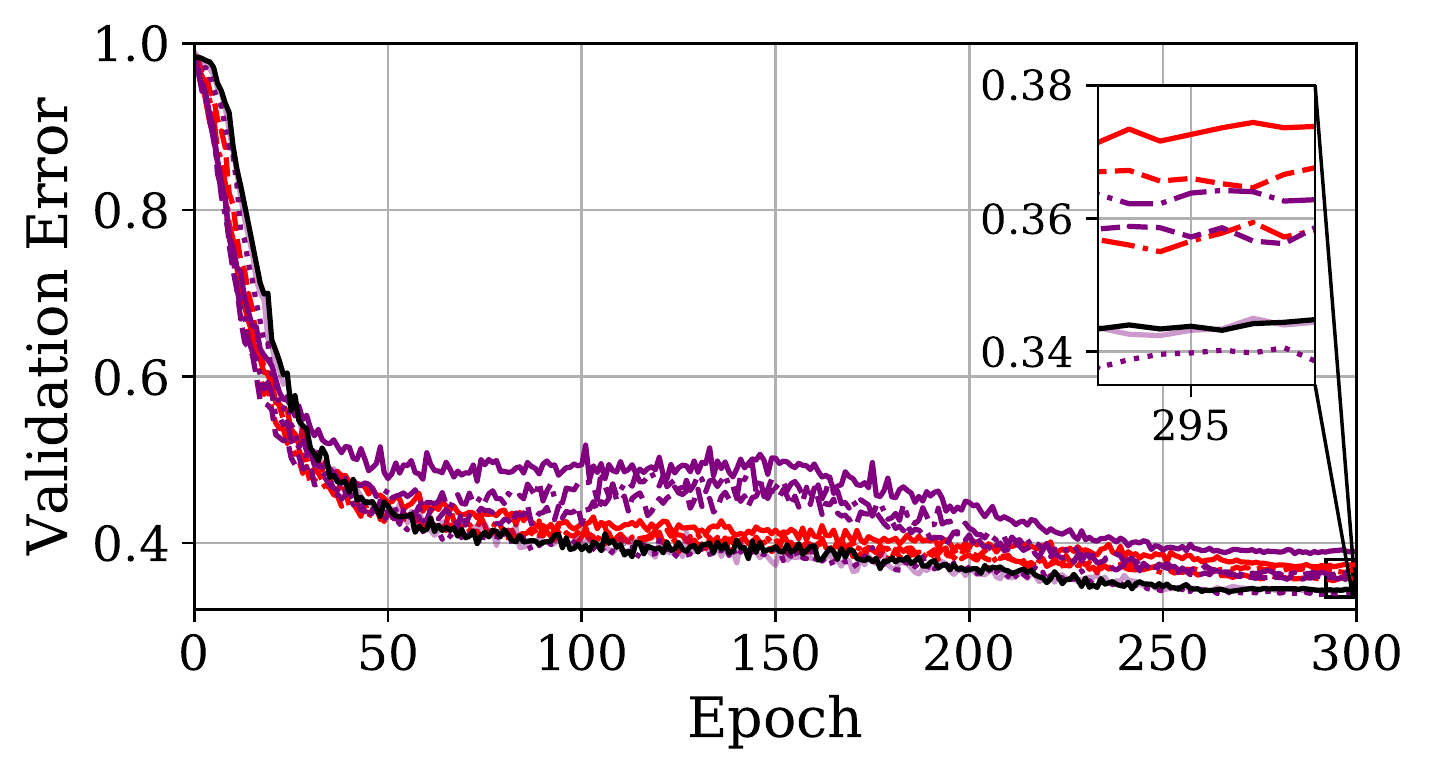}
	\end{subfigure}
	\vspace{-5pt}
	\caption{Training/validation error of the Adam optimiser for VGG-$16$ on the CIFAR-$100$ dataset with various learning rates $\alpha$ and damping values, $\delta$.}
	\label{fig:adamlrandeps}
\end{figure*}

\label{sec:bnexp}
\begin{figure*}[h!]
	\centering
	\begin{subfigure}[b]{0.61\textwidth}
		\includegraphics[width=\textwidth]{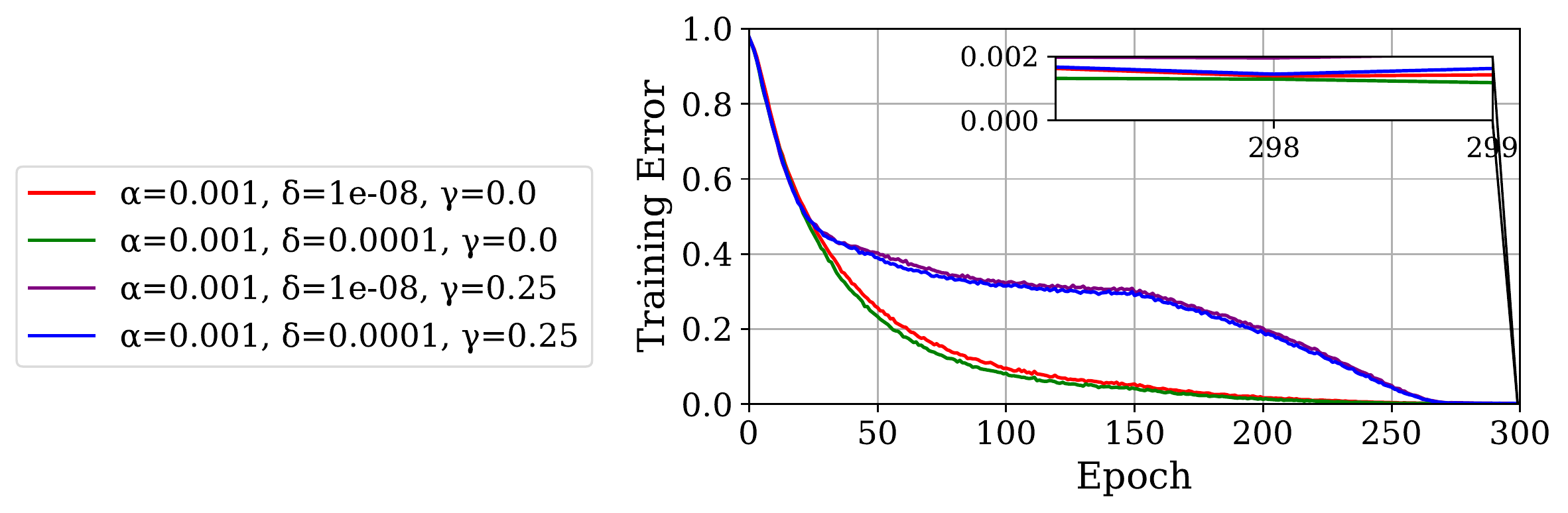}
	\end{subfigure}
	\begin{subfigure}[b]{0.37\textwidth}
		\includegraphics[width=\textwidth]{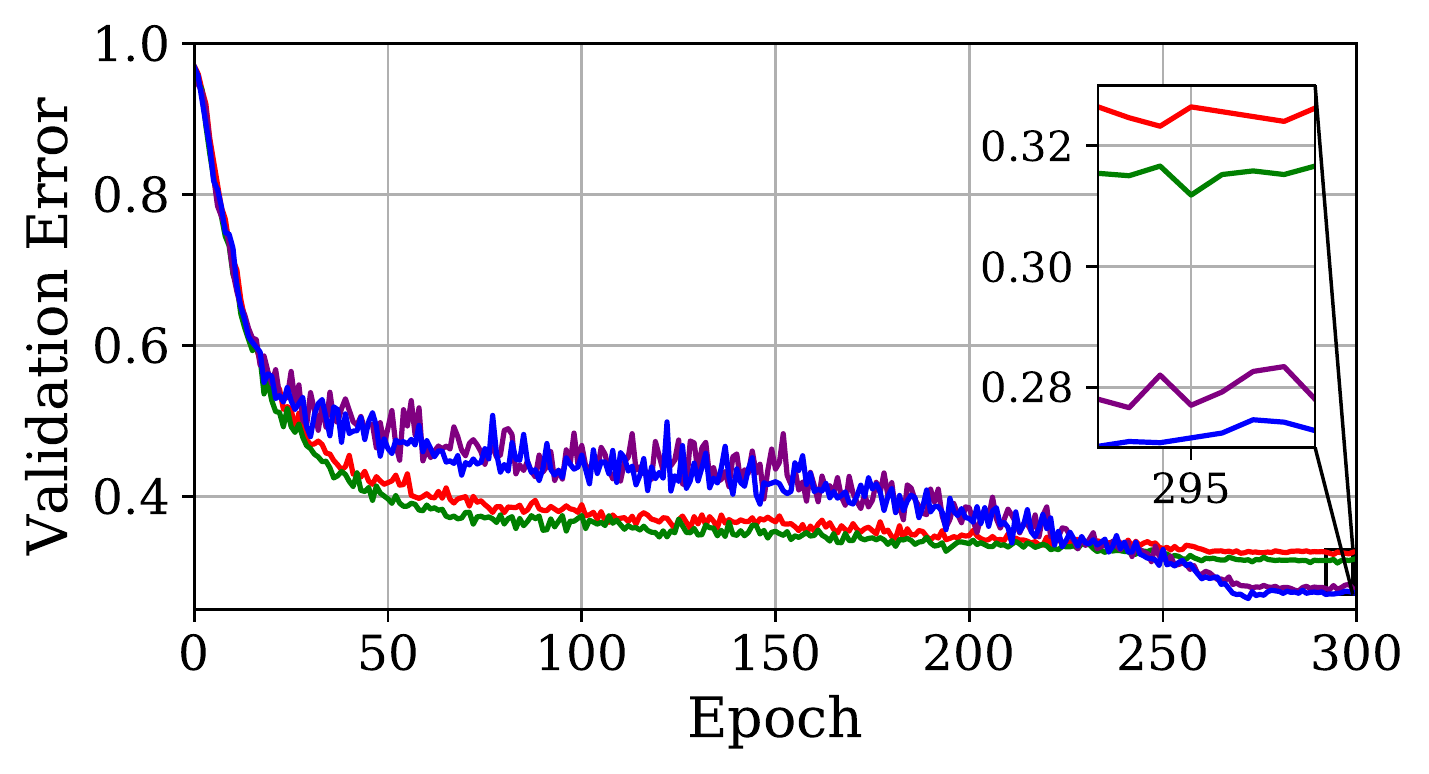}
	\end{subfigure}
	\caption{Training/validation error of the Adam optimiser for VGG-$16$BN using Batch Normalisation and Decoupled Weight Decay on the CIFAR-$100$ dataset with various learning rates $\alpha$ and damping values, $\delta$.}
	\label{fig:adambnlrandeps}
\end{figure*}

\paragraph{ResNet-$50$ ImageNet.}

\begin{figure}
\begin{minipage}[t]{1\textwidth}
		\centering
		\begin{subfigure}[b]{0.49\textwidth}
			\includegraphics[width=\textwidth]{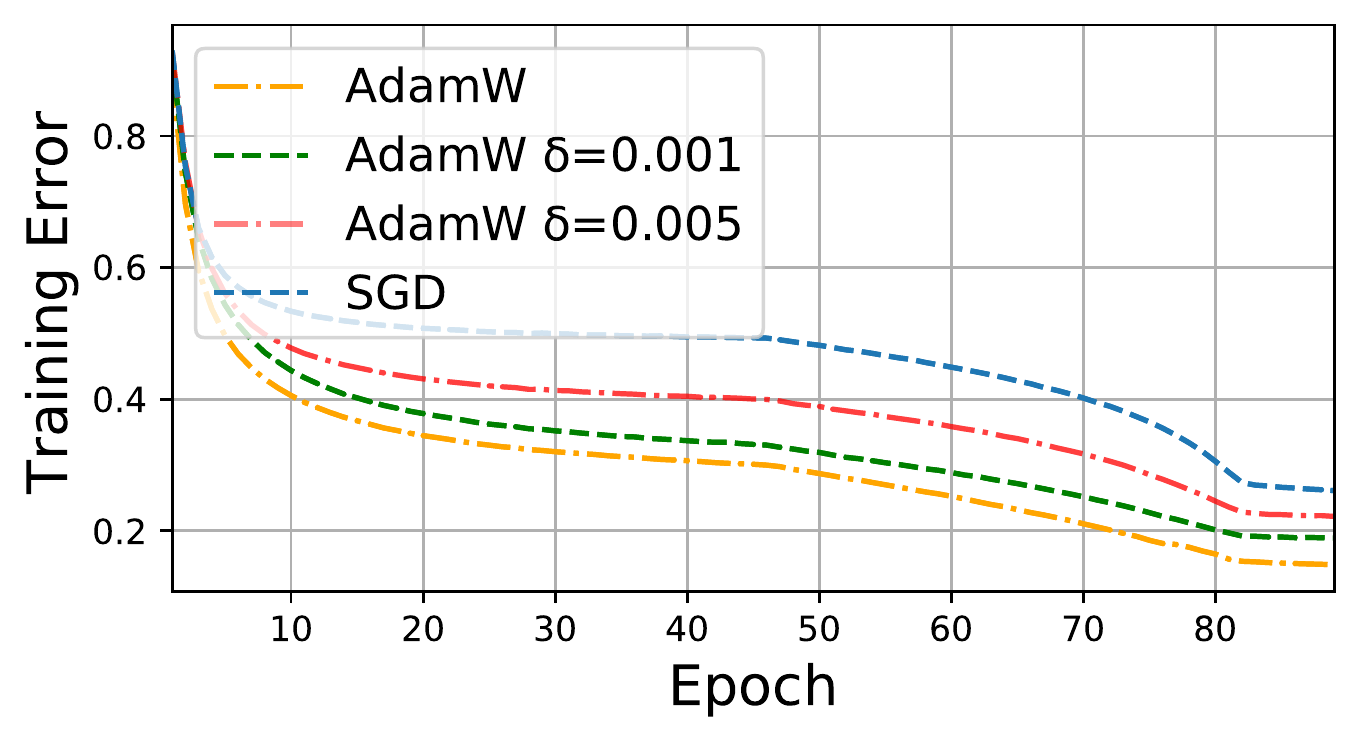}
			\caption{ResNet-$50$ Training Error}
			\label{subfig:r50train}
		\end{subfigure}
		\begin{subfigure}[b]{0.49\textwidth}
			\includegraphics[width=\textwidth]{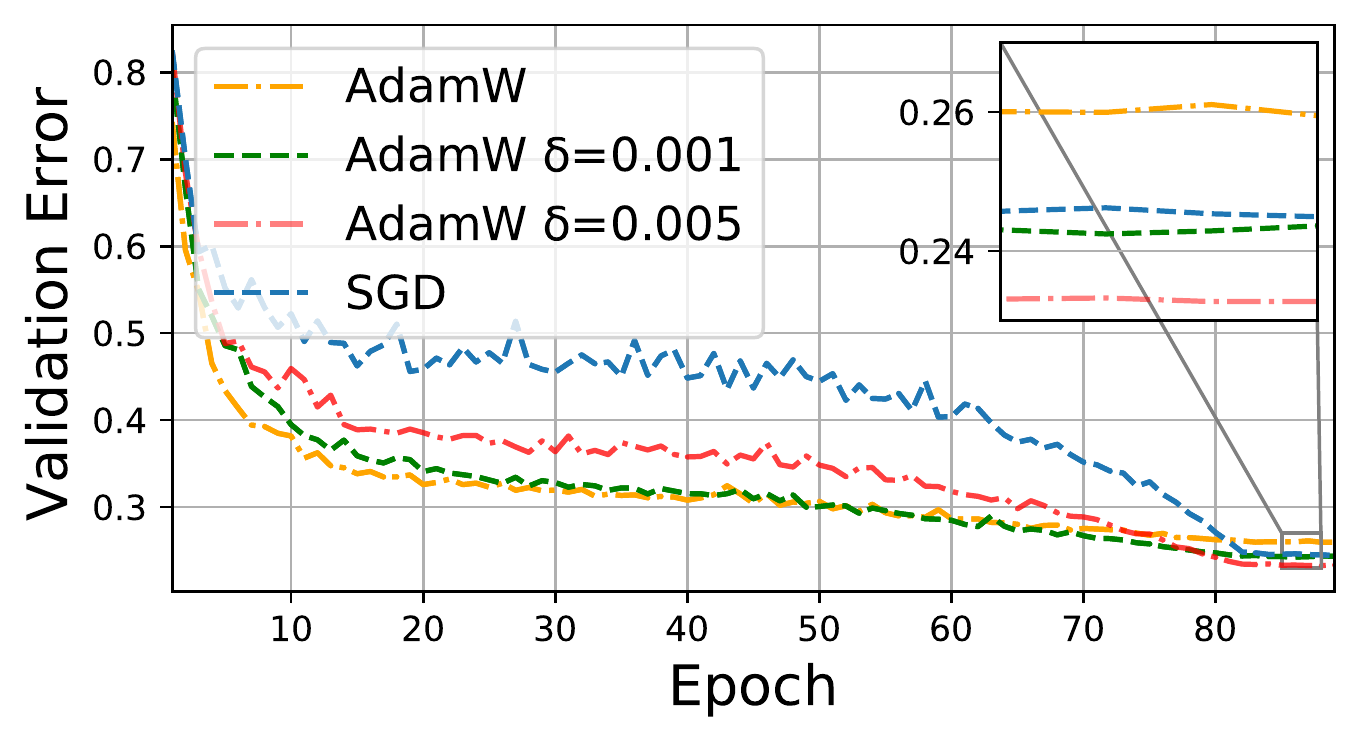}
			\caption{ResNet-$50$ Testing Error}
			\label{subfig:r50test}
		\end{subfigure}
\end{minipage}

\label{fig:res50adamw}

\caption{(a-b) The influence of $\delta$ on the generalisation gap. Train/Val curves for ResNet-$50$ on ImageNet. The generalisation gap is completely closed with an appropriate choice of $\delta$.}
\end{figure}

As shown in Fig.~\ref{subfig:r50train},\ref{subfig:r50test}, these procedures have practical impact on large scale problems. Here we show that under a typical $90$ epoch ImageNet setup~\cite{he2016deep}, with decoupled weight decay $0.01$ for AdamW and $0.0001$ for SGD, that by increasing the numerical stability constant $\delta$ the generalisation performance can match and even surpass that of SGD, which is considered state-of-the-art and beats AdamW without $\delta$ tuning by a significant margin.

\begin{table}[h]
		\begin{tabular}{@{}llllll@{}}
		\toprule
		% 		\textbf{Dataset/Model} & \textbf{SGD} & \textbf{Adam} $\delta=5e^{-3}$ & \textbf{Adam} $\delta=1e^{-8}$ \\ \midrule
		\textbf{Dataset} & \textbf{Classes}& \textbf{Model Architecture} & \textbf{SGD} & \textbf{Adam-D} & \textbf{Adam} \\ \midrule
		% 		\textbf{Best} & 65.4 +/- 0.7 & 65.8 +/- 0.6 & 62.2 +/- 0.4 \\
		CIFAR-100 & 100 & VGG16  & 65.3 $\pm$ 0.6 & 65.5 $\pm$ 0.7 & 61.9 $\pm$ 0.4 \\
		\midrule
		
		%\bottomrule
		ImageNet & 1000& ResNet50 & 75.7 $\pm$ 0.1 & 76.6 $\pm$ 0.1 & 74.04* \\ \bottomrule
	\end{tabular}
	\vspace{15pt}
	\caption{Statistical Significance. Comparison of test accuracy across CIFAR 100 (5 seeds) and ImageNet (3 seeds). \textbf{Adam-D} denotes Adam with increased damping ($\delta=5e^{-3}$ for CIFAR-100, $\delta=1e^{-4}$ for ImageNet). *Since it is well established that Vanilla Adam does not generalise well for ImageNet, we do not run this experiment for multiple seeds, we simply report a single seed result for completeness. A more complete discussion for Adam and its generalisation in vanilla form can be found in \cite{granziol2020iterate}.}
	\label{tab:seeds}
\end{table}

\section{Optimal adaptive damping from random matrix theory} \label{sec:adaptive}

% \label{sec:linshrink}
Recall the scaling applied in the direction of the $i^{\mathrm{th}}$ eigenvector in (\ref{eq:secondorderopt}). We make the following observation
\begin{equation}
	\begin{aligned}
		\label{eq:diegosderivation}
		& \frac{1}{\lambda_{i}+\delta} = \frac{1}{\beta\lambda_{i}+(1-\beta)}\cdot\frac{1}{\kappa} \\
	\end{aligned}
\end{equation}
%\nickcomment{If this correspondence with shrinkage isn't just a fluke, we would expect alternate shrinkage strategies to also work, such as number 3 or 5 in  \cite{bun2016my}. Particularly as the linear shrinkage is only optimal for the multiplicative noise model.}
where $\kappa = \beta^{-1}, \beta = (1+\delta)^{-1}$. Hence, using a damping $\delta$ is formally equivalent to applying linear shrinkage with factor $\beta=(1+\delta)^{-1}$ to the estimated Hessian and using a learning rate of $\alpha\beta$.
% estimation.
\nick{Shrinkage estimators are widely used in finance and data science, with linear shrinkage being a common simple method applied to improve covariance matrix estimation \cite{ledoit2004well}. The practice of shrinking the eigenvalues while leaving the eigenvectors unchanged is well-established in the fields of sparse component analysis and finance \cite{bun2017cleaning}. \cite{bun2016rotational} even show that under the free multiplicative noise model, the linear shrinkage estimator,
% \begin{equation}
$
\Tilde{\mH} = \beta \mH + (1-\beta)\mI = \argmin_{\mH^*}||\mH^*-\mH_{\mathrm{true}}||_2
% \nonumber
$, 
% \end{equation}
gives the minimum error between the estimator and the true Hessian, and an explicit expression for the optimal $\beta$ is found depending only on the dimensionality of the model and the noise variance. In our optimisation context, the linear shrinkage estimator is unlikely to be optimal, however it has the great advantage of being simple to integrate into existing adaptive optimisers and it acts intuitively to reduce the movement of the optimiser in pure-noise directions. 
Our interpretation reveals that the damping parameter should not be viewed as a mere numerical convenience to mollify the effect of very small estimate eigenvalues, but rather that an optimal $\delta$ should be expected, representing the best linear approximation to the true Hessian and an optimal balancing of variance (the empirical Hessian) and bias (the identity matrix). This optimal choice of $\delta$ will produce an optimiser that more accurately descends the directions of the true loss.}

The linear shrinkage interpretation given by (\ref{eq:diegosderivation}) is an elementary algebraic relation but does not by itself establish any meaningful link between damping of adaptive optimisers and linear shrinkage estimators. To that end, we return to the random matrix model (\ref{eq:additive_noise}) for the estimated Hessian:
Let us write the Hessian as \begin{align*}
    \mH_{\text{batch}} = \mH_{\text{true}} + \mX
\end{align*}
where $\mX$ is a random matrix with $\mathbb{E}\mX = 0$. Note that this model is entirely general, we have simply defined $\mX = \mH_{\text{batch}}  - \mathbb{E}\mH_{\text{batch}} $ and $\mathbb{E}\mH_{\text{batch}}  = \mH_{\text{true}}$. We then seek a linear shrinkage estimator $\tilde{\mH}(\beta) = \beta \mH_{\text{batch}} + (1-\beta)\mI$ such that $E(\beta)= P^{-1}\Tr (\tilde{\mH} - \mH_{\text{true}})^2$ is minimised. Note that this is the same objective optimised by \cite{bun2016rotational} to obtain optimal estimators for various models. In this context, we are not finding the optimal estimator for $\mH_{\text{true}}$ but rather the optimal \emph{linear shrinkage} estimator. We have \begin{align*}
    E(\beta) = \frac{1}{P}\Tr\left[(\beta-1) \mH_{\text{true}} + \beta\mX + (1-\beta)\mI\right]^2 \equiv \frac{1}{P}\Tr \left[  (\beta - 1)\mH_{\text{true}}+ \mY_{\beta}\right]^2
\end{align*}
where $\mY_{\beta} = \beta\mX + (1-\beta)\mI$.

A natural assumption in the case of deep learning is that $\mH_{\text{true}}$ is low-rank, i.e. for $P\rightarrow\infty$ either $\text{rank}(\mH_{\text{true}}) = r$ is fixed or  $\text{rank}(\mH_{\text{true}}) = o(P)$. Empirical evidence for this assumption is found in \cite{granziol2020learning,sagun2016eigenvalues,sagun2017empirical,papyan2018full,ghorbani2019investigation}. In this case the bulk of the spectrum of $\mY_{\beta}$ is the same as that of $(\beta-1) \mH_{\text{true}} + \mY_{\beta}$ \cite{benaych2011eigenvalues,capitaine2016spectrum,belinschi2017outliers}. We will also assume that $\mX$ admits a deterministic limiting spectral measure $\mu_X$ such that \begin{align}
    \frac{1}{P}\sum_{j=1}^P \delta_{\lambda(X)_i} \rightarrow \mu
\end{align}
weakly almost surely. Say $\omega_X(x) dx = d\mu(x)$. Then $\mY_{\beta}$ has limiting spectral density \begin{align*}
    \omega_Y(y) = \beta^{-1}\omega_X(\beta^{-1}(y - 1 + \beta)).
\end{align*}
Then for large $P$ 
\begin{align*}
    E(\beta) \approx \beta^{-1}\int y^2 \omega_X(\beta^{-1}(y - 1 + \beta)) ~dy &=  \int (\beta x + 1 - \beta)^2 \omega_X(x) ~ dx \\
    &=  \beta^2 \mu_X(x^2) + (1-\beta)^2
\end{align*}
as the centred assumption on $\mX$ means that $\int x\omega_X(x) ~dx = 0$. $\mu_X(x^2)$ is shorthand for $\int x^2\omega_X(x)~ dx$. $E(\beta)$ is thus minimised to leading order at $\beta = (1 + \mu_X(x^2))^{-1}$. Recalling that $\beta^{-1} = (1+\delta)^{-1}$, this yields $\delta = \mu_X(x^2)$ i.e. the optimal level of damping at large finite $P$ is approximately
\begin{align}\label{eq:finaloptdamp}
\delta=P^{-1}\Tr \mX^2.
\end{align}
Note that the value (\ref{eq:finaloptdamp}) is a very natural measure of the Hessian noise variance. Therefore if the random matrix model described above is appropriate and the linear shrinkage interpretation (\ref{eq:diegosderivation}) is meaningful we should expect it to result in close to optimal performance of a given adaptive optimiser. The purpose of adaptive optimisers is to accelerate training, in part by allowing for larger stable learning rates. As discussed throughout this paper, such optimisation speed often comes at the cost of degraded generalisation. In this context, `optimal performance' of adaptive optimisers should be taken to mean fast training and good generalisation. 
% \begin{hypothesis}\label{hyp:delta}
% Setting the damping $\delta$ in an adaptive optimiser using (\ref{eq:finaloptdamp}) gives optimal estimation of the true Hessian amongst all values of $\delta$ and hence gives the greatest rate of loss decrease and best generalisation amongst all values of $\delta$.
% \end{hypothesis}
As we have discussed above, very large values of $\delta$ recover simple non-adaptive SGD, so using (\ref{eq:finaloptdamp}) we should be able to obtain generalisation performance at least as good as SGD and faster optimisation than any choice of $\delta$ including the default very small values often used and the larger values considered in Section \ref{sec:nnexperiments}. 

The value of (\ref{eq:finaloptdamp}) can be easily learned by estimating the variance of the Hessian. The Hessian itself cannot be computed exactly, as it is far too large for $P \geq O(10^7)$, however one can compute $\mH \vv$ (and hence $\mH^2 \vv$) for any vector $\vv$, using $\nabla^2 L \vv = \nabla (\vv^T\nabla L)$. The full approach is given in Algorithm \ref{alg:hessvar}.
\begin{algorithm}[H]
	\begin{algorithmic}[1]
		\STATE {\bfseries Input:} Sample Hessians $\mH_{i}\in \mathbb{R}^{P\times P}$, $1\leq i < N$
		\STATE {\bfseries Output:} Hessian Variance $ \sigma^{2}$
		\STATE $\vv \in \mathbb{R}^{1\times P} \sim \mathcal{N}(\boldsymbol{0}, \mI)$
		\STATE Initialise $\sigma^{2}=0, i = 0$, $\vv \leftarrow \vv/||\vv||$
		\FOR{$i < N$}
		\STATE $\sigma^{2} \leftarrow \sigma^{2} + \vv^{T}\mH_{i}^{2}\vv$
		\STATE $i \leftarrow i + 1$
		\ENDFOR
		\STATE $\sigma^{2} \leftarrow \sigma^{2} - [\vv^{T}(1/N\sum_{j=1}^{N}\mH_{j})\vv]^{2}$
		%		\STATE Moments $\{ \mu_i \}_{i=1}^m \leftarrow$ STE $\left(\mL_{\mathrm{norm}}, d, m \right)$
		%		\STATE $\{ \alpha_i \}_{i=1}^m \leftarrow$ MaxEnt algorithm $\left(\{ \mu_i \}_{i=1}^m \right)$
		%		\STATE EGS $p(\lambda) = \exp[-(1+\sum_{i}\alpha_{i}\lambda^{i})]$
	\end{algorithmic}
	\caption{Algorithm to estimate the Hessian variance}
	\label{alg:hessvar}
\end{algorithm}

\subsection{Experimental Design and Implementation Details}

In order to test our hypothesis for the derived optimal $\delta$ (\ref{eq:finaloptdamp}), we run the classical VGG network \cite{simonyan2014very} with $16$ layers on the CIFAR-$100$ dataset, without weight decay or batch normalisation. This gives us maximal sensitivity to the choice of learning rate and appropriate damping. 

\medskip
Now in practice the damping coefficient is typically grid searched over several runs \cite{dauphin2014identifying} or there are heuristics such as the Levenberg–Marquardt to adapt the damping coefficient \cite{martens2015optimizing}, which however we find does not give stable training for the VGG. We hence compare against a fixed set damping value $\delta$ and a learned damping value as given by our equation (\ref{eq:finaloptdamp}). We find that the variance of the Hessian (\ref{eq:finaloptdamp}) at a random point in weight space (such as at initialisation) or once network divergence has occurred is zero, hence the initial starting value cannot be learned as, with a damping of near zero, the network entirely fails to train (no change in training loss from random). This is to be expected, as in this case the local quadratic approximation to the loss inherent in adaptive methods breaks down. Hence we initialise the learning algorithm with some starting value $\delta^{*}$, which is then updated every $100$ training iterations using equation (\ref{eq:finaloptdamp}). Strictly speaking we should update every iteration, but the value of $100$ is chosen arbitrarily as a computational efficiency. Since we are using the variance of the Hessian, which is expensive to compute compared to a simple gradient calculation, we do not want to compute this quantity too often if it can be helped. We run our experiments on a logarithmic grid search in near factors of $3$. So learning rates and damping rates, either flat or learned are on the grid of $0.0001,0.0003,0.001...$.

We find under this setup that the time taken per epoch against the flat damping schedule is only doubled. We get identical results for using a damping gap of $10$ and so do not consider this to be a very relevant hyper-parameter. We further calculate the variance of the Hessian over a sub-sample of $10000$ examples and do not calculate the variance sample by sample, but over batches of $128$ to speed up the implementation. Under the assumption that the data is drawn i.i.d from the dataset the variance is simply reduced by a factor $(\frac{1}{B}-\frac{1}{N}) \approx \frac{1}{B}$ for a small batch size. We do not consider the impact of using only a sub-sample of the data for estimation, but we expect similar results to hold compared to the enitre dataset as long as the sub-sample size $S\gg B$. This should allow such a method to be used even for very large datasets, such as ImageNet (with $1$-million images), for which a pass of the entire dataset is extremely costly. In theory the sub-sample size and mini-batch size for Hessian variance estimation could be two hyper-parameters which are tuned by considering the effect of reduction on training set or validation set loss metrics with the trade off for computational cost. We do not conduct such analysis here.

We also incorporate an exponential moving average into the learned damping with a co-efficient of $0.7$\footnote{This value is not tuned and in fact from our plots it may be advisable to consider higher values for greater stability} to increase the stability of the learned damping.

\subsection{Experiment on CIFAR-100 using KFAC to validate the optimal linear shrinkage}

\medskip
For large damping values $\delta$ we simply revert to SGD with learning rate $\alpha/\delta$, so we follow the typical practice of second order methods and use a small learning rate and correspondingly small damping coefficient.  However as shown in Figure \ref{fig:adamsgd} the generalisation and optimisation are heavily dependent on the global learning rate, with larger learning rates often optimising less well but generalising better and vice versa for smaller learning rates. We hence investigate the impact of our damping learner on learning rates one order of magnitude apart. Where in the very low learning rate regime, we show that our method achieves singificantly improved training stability with low starting damping and fast convergence and for the large learning rate regime that we even exceed the SGD validation set result.

\paragraph{Training KFAC with Auto-Damping:}
We show the results for a global learning rate of $0.0001$ in Figure \ref{subfig:damperr}. We see that for the flat damping methods with low values of damping, that training becomes unstable and diverges, despite an initially fast start. Higher damped methods converge, but slowly. In stark contrast, our adaptive damping method is relatively insensitive to their chosen initial values. We show here $\delta^{*} = \alpha,3\alpha,10\alpha$ and all converge and moreover significantly faster than all flat damping methods. The smaller initial damping coefficients $\delta = \alpha,3\alpha$ converge faster than the larger and, interestingly, follow very similar damping trajectories throughout until the very end of training, as shown in Figure \ref{subfig:damping}.

\begin{figure}[!h]
	\begin{subfigure}[b]{0.66\textwidth}
		\includegraphics[width=\textwidth]{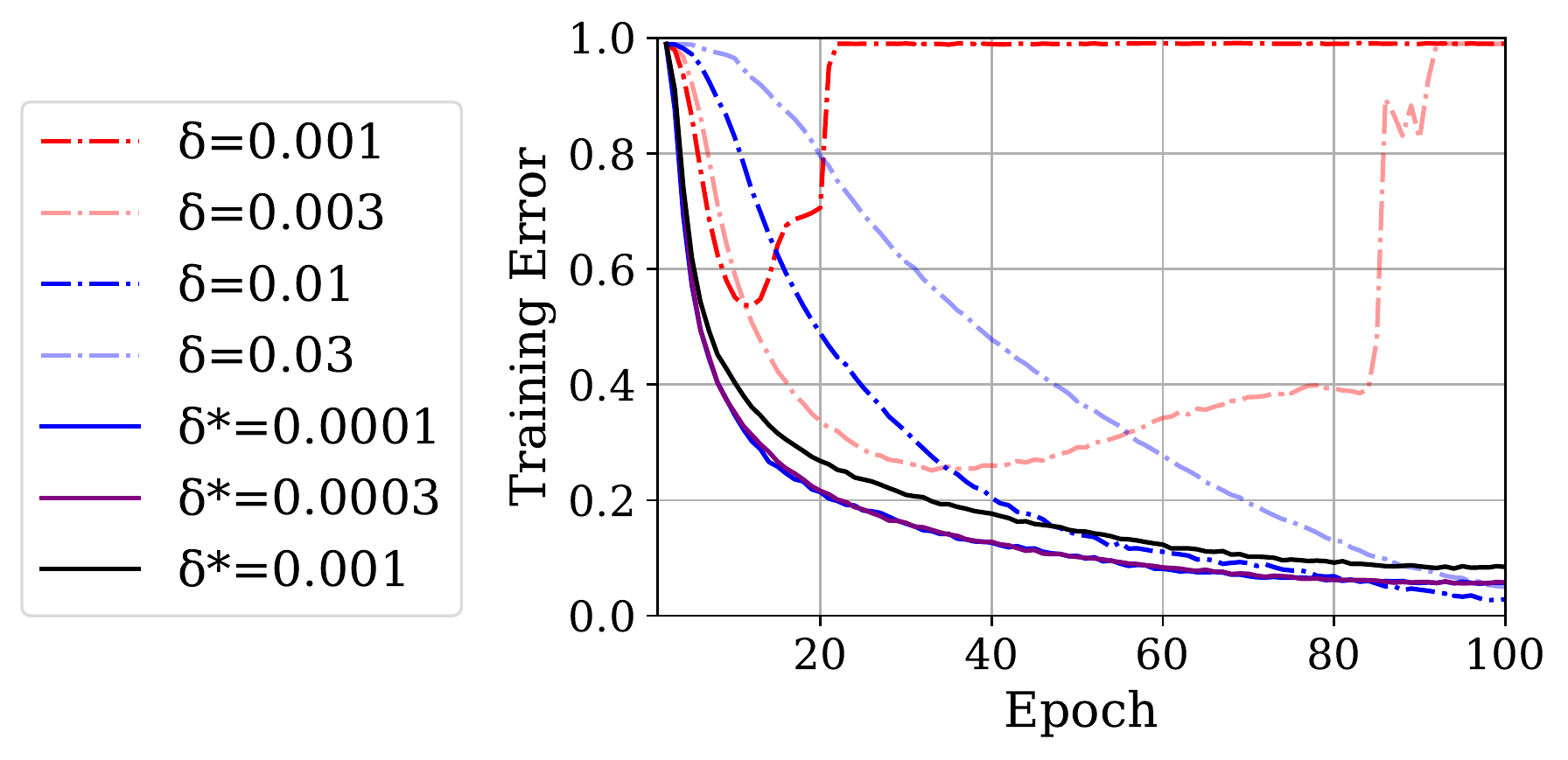}
		\caption{Training Error as a function of Epoch}
		\label{subfig:damperr}
	\end{subfigure}
	\begin{subfigure}[b]{0.33\textwidth}
		\includegraphics[width=\textwidth]{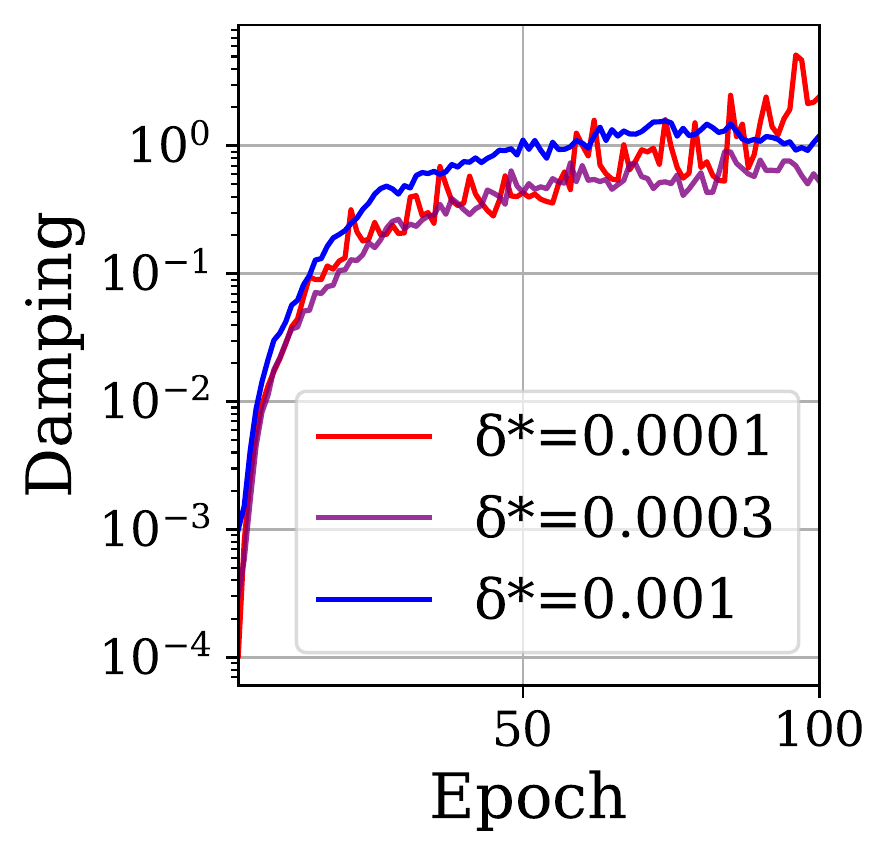}
		\caption{Damping per Epoch}
		\label{subfig:damping}
	\end{subfigure}
	\caption{VGG-$16$ on CIFAR-$100$ dataset using the KFAC optimiser with $\gamma=0$ (no weight decay) for a learning rate of $\alpha=0.0001$, batch size $B=128$ and damping set by $\delta$. For adaptive damping methods the damping is given an initial floor value of $\delta^{*}$ and is then updated using the variance of the Hessian every $100$ steps.}
\end{figure}

\paragraph{Getting Great Generalisation with KFAC and Auto-Damping:}
We similarly train KFAC on the VGG-$16$ with a larger learning rate of $0.001$, in order to achieve better generalisation. Here we see in Figure \ref{subfig:trainkfac} that relatively low values of flat damping such as $0.01$ and  $0.03$ very quickly diverge, wheras a large value of $0.1$ converges slowly to a reasonable test error. The corresponding learned damping curves of $0.01$ and $0.03$ however converge quickly and the $0.03$ initialised damping curve even beats the generalisation performance of the large flat damped version and the test result of SGD on $3$x as many training epochs.
\begin{figure}[!h]
	\begin{subfigure}[b]{0.58\textwidth}
		\includegraphics[width=\textwidth]{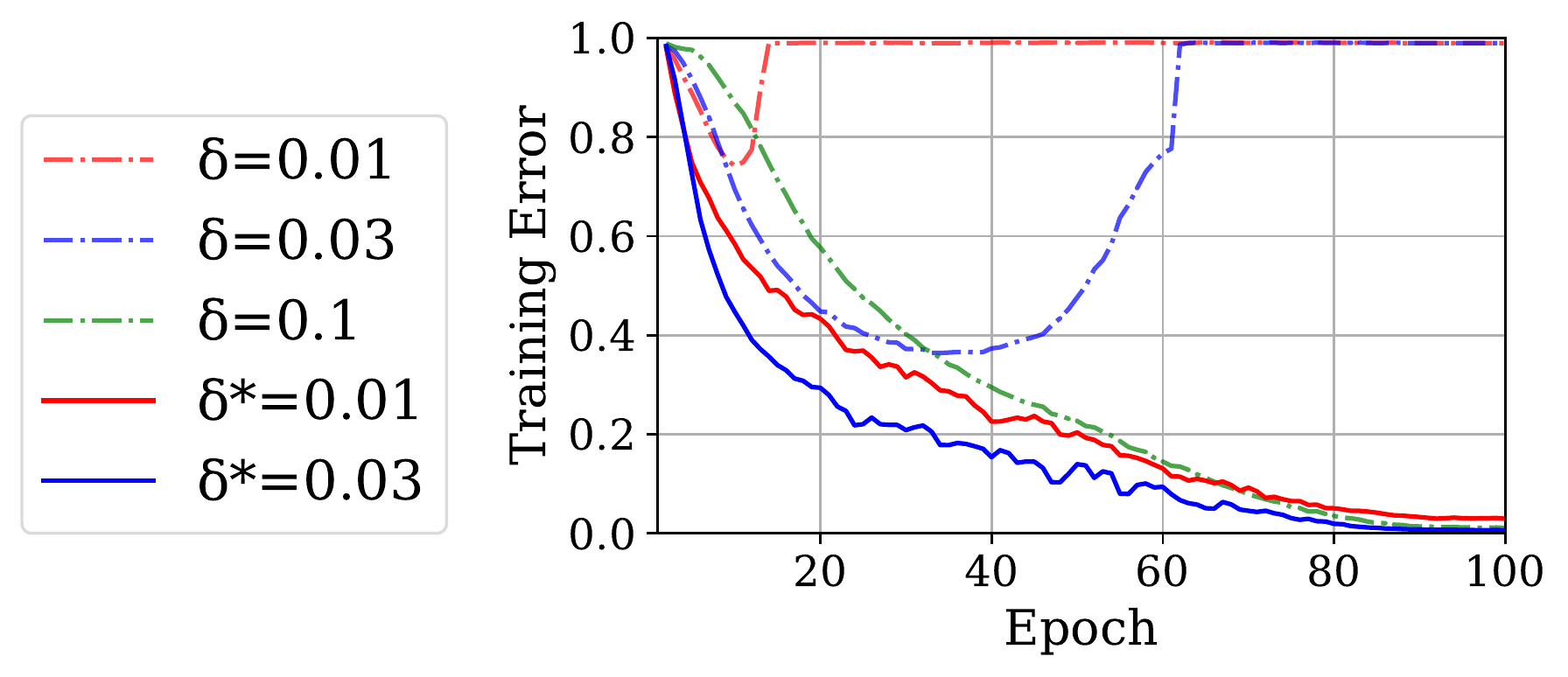}
		\caption{Training Error}
		\label{subfig:trainkfac}
	\end{subfigure}
	\begin{subfigure}[b]{0.41\textwidth}
		\includegraphics[width=\textwidth]{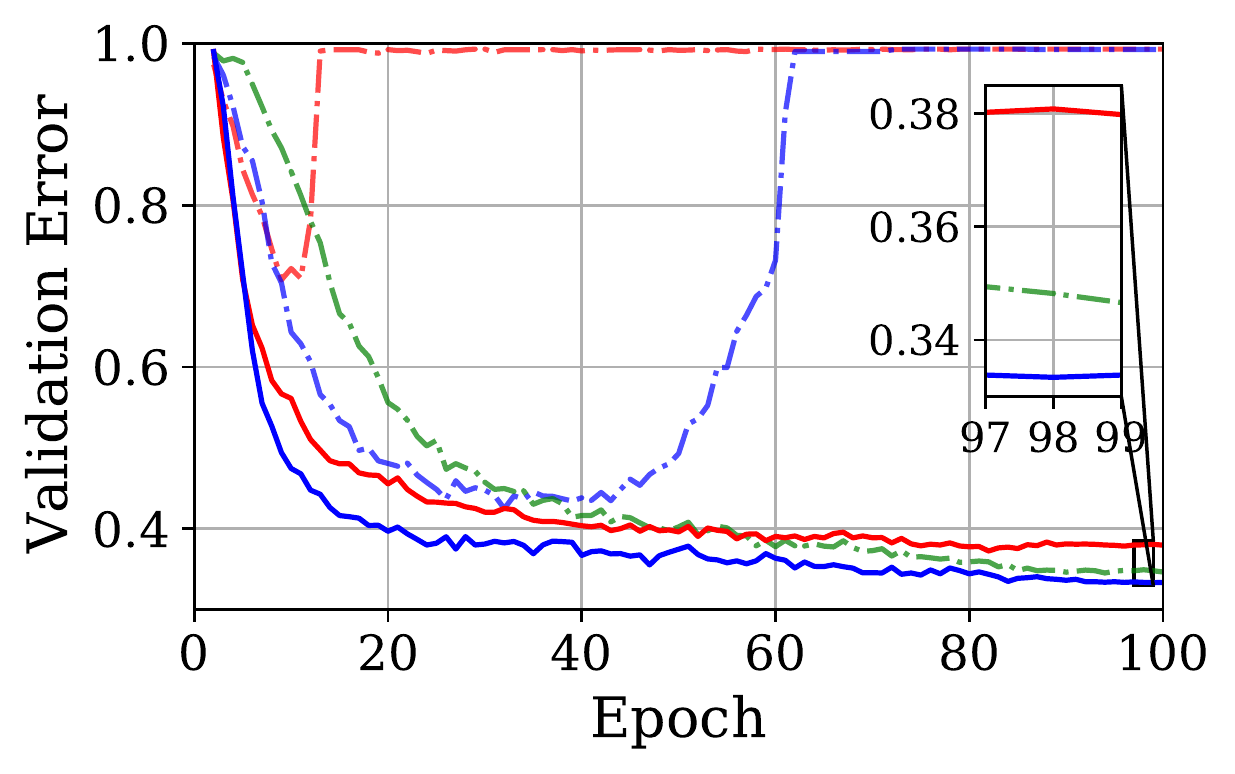}
		\caption{Val Error}
		\label{subfig:valkfac}
	\end{subfigure}
	\caption{VGG-$16$ on CIFAR-$100$ dataset using the KFAC optimiser with $\gamma=0$ (no weight decay) for a learning rate of $\alpha=0.001$, batch size $B=128$ and damping set by $\delta$. For adaptive damping methods the damping is given an initial floor value of $\delta^{*}$ and is then updated using the variance of the Hessian every $100$ steps.}
\end{figure}
\paragraph{A further look at the value of adaptive damping}
\begin{figure}[!h]
	\begin{subfigure}[b]{0.46\textwidth}
		\includegraphics[width=\textwidth]{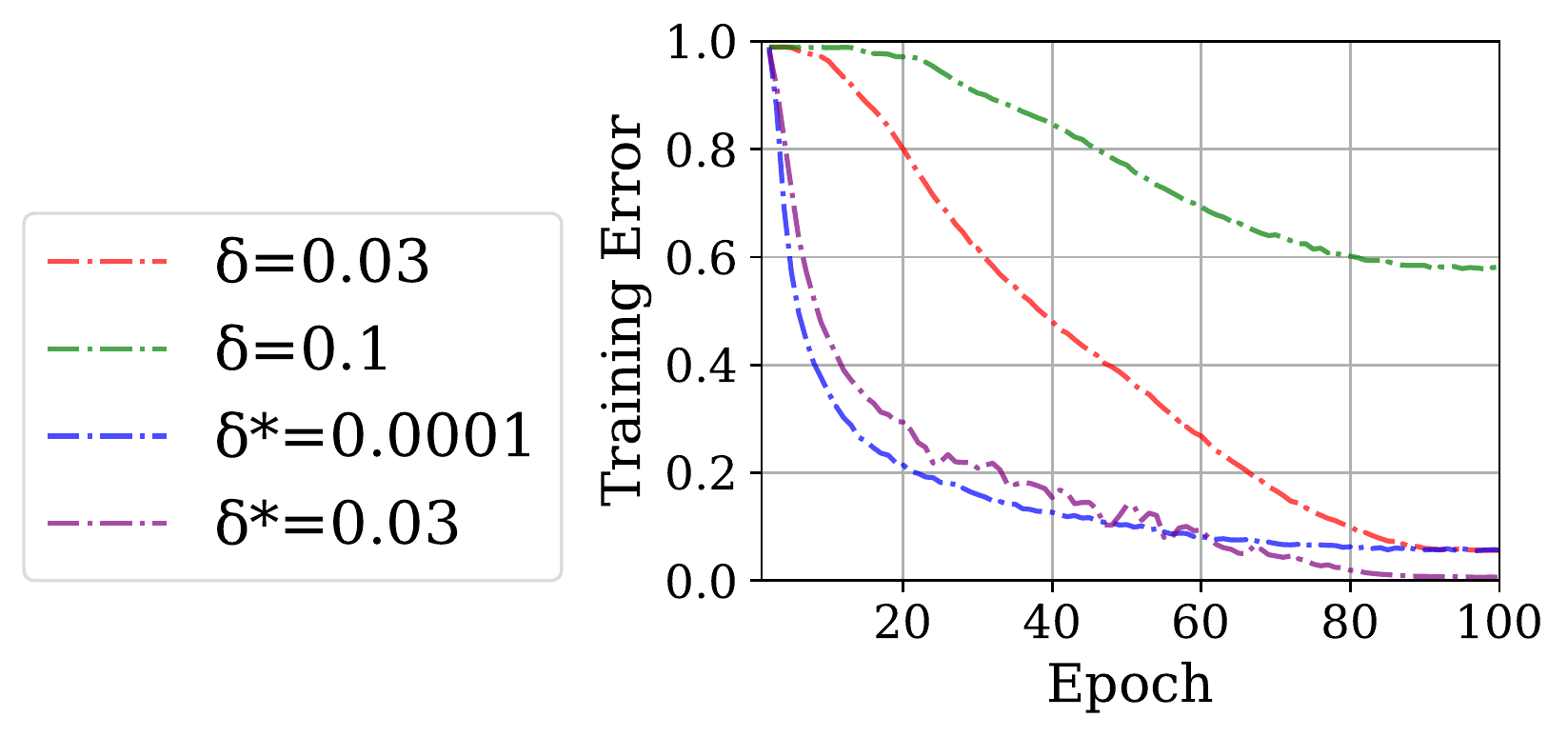}
		\caption{Training Error}
		\label{subfig:trainkfacclean}
	\end{subfigure}
	\begin{subfigure}[b]{0.29\textwidth}
		\includegraphics[width=\textwidth]{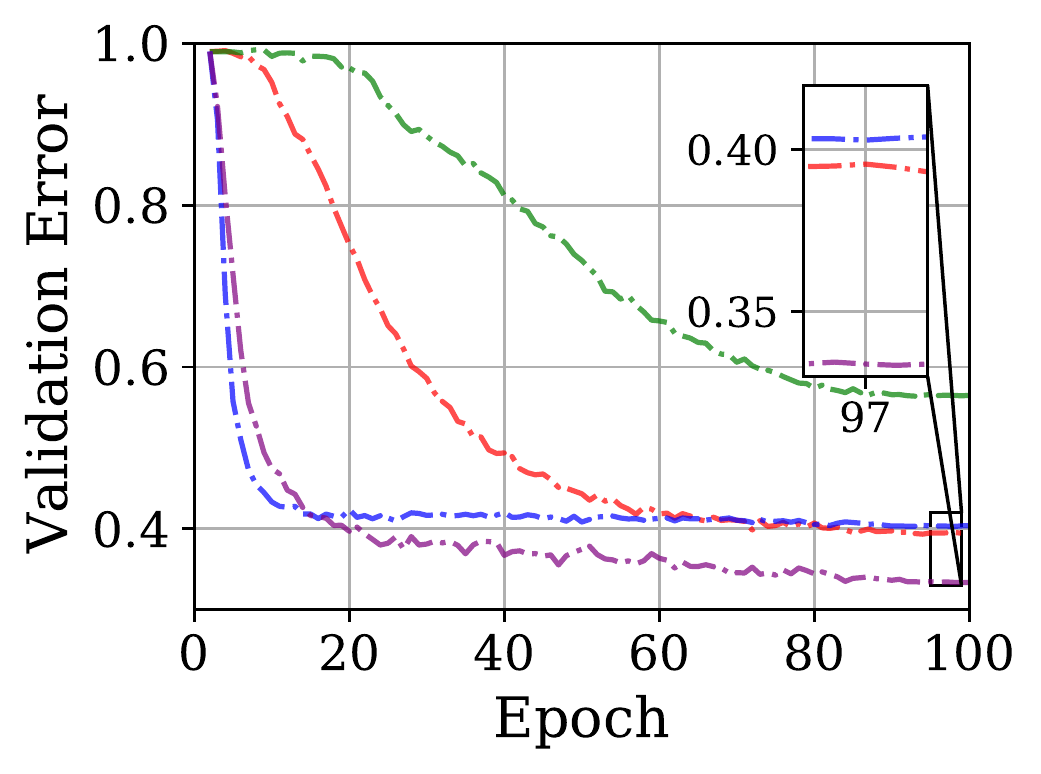}
		\caption{Val Error}
		\label{subfig:valkfacclean}
	\end{subfigure}
	\begin{subfigure}[b]{0.235\textwidth}
		\includegraphics[width=\textwidth]{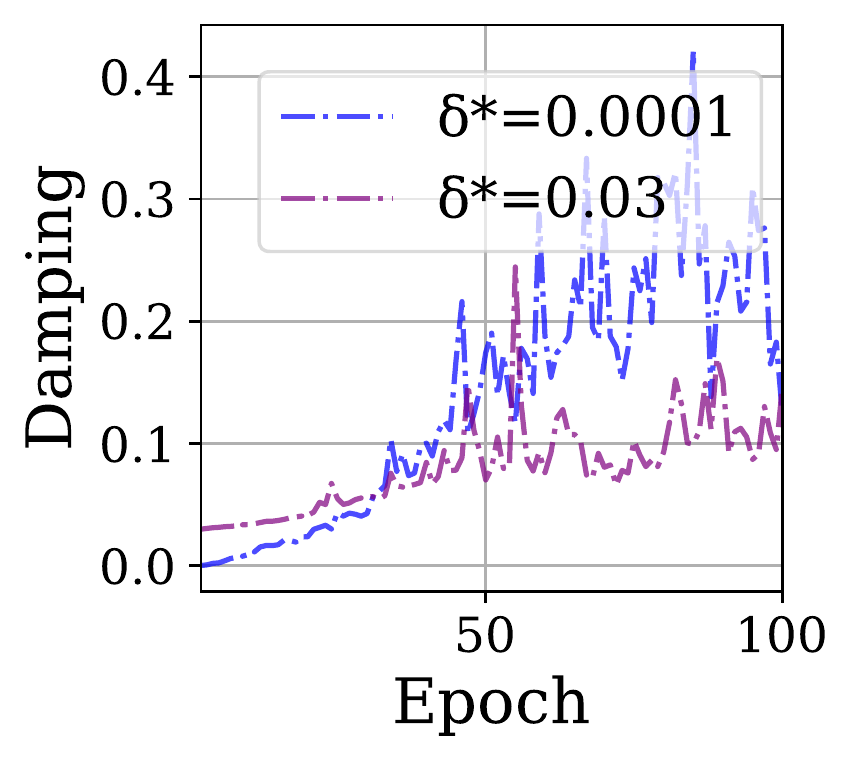}
		\caption{Damping}
		\label{subfig:dampkfac}
	\end{subfigure}
	\caption{VGG-$16$ on CIFAR-$100$ dataset using the KFAC optimiser with $\gamma=0$ (no weight decay) for a learning rate of $\alpha=0.0001$, batch size $B=128$ and damping set by $\delta$. For adaptive damping methods the damping is given an initial floor value of $\delta^{*}$ and is then updated using the variance of the Hessian every $100$ steps.}
	\label{fig:kfacclean}
\end{figure}
To elucidate the impact and workings of the adaptive damping further, we consider a select set of curves the learning rate of $\alpha = 0.0001$, shown in Figure \ref{fig:kfacclean}. here we see that starting with an initial damping of $\delta=\alpha$, the adaptive method reaches a comparable generalisation score to the flat damping of $\delta=0.03$ but at a much faster convergence rate. The initial damping of $\delta=0.03$ converges not quite as quickly but trains and generalises better than its lower starting damping counterpart. Note from Figure \ref{subfig:dampkfac} that even though the damping of this curve reaches $\approx 0.1$ that starting with a flat damping of $0.1$ never achieves a comparable generalisation (or even trains well). This implies as expected that it is important to adjust damping during training.

\subsection{Adam with Auto-Damping}
Given that Adam does not employ an obvious curvature matrix, it is curious to consider whether our learned damping estimator can be of practical value for this optimiser. As discussed in the previous section, Adam's implied curvature can be considered a diagonal approximation to the square root of the gradient covariance. The covariance of the gradients has been investigated to have similarities to the Hessian \cite{jastrzebski2020the} and we further show in Section \ref{sec:quality} experimentally that adaptive methods do learn information about the sharpest eigenvalue/eigenvector pairs. However the nature of the square root, derived from the regret bound in \cite{duchi2011adaptive} presents an interesting dillemna. In the case of very very small eigenvalues of $\mB$, the square root actually reduces their impact on the optimisation trajectory, hence it is very plausible that the learned damping could be too harsh (as it is expected to work optimially for the eigenvalues of $\mH$ and not $\sqrt{\mH})$.  This is actually exactly what we see in Figure \ref{fig:adamautodamp}.
\begin{figure}[h!]
	\begin{subfigure}[b]{0.62\textwidth}
		\includegraphics[width=\textwidth]{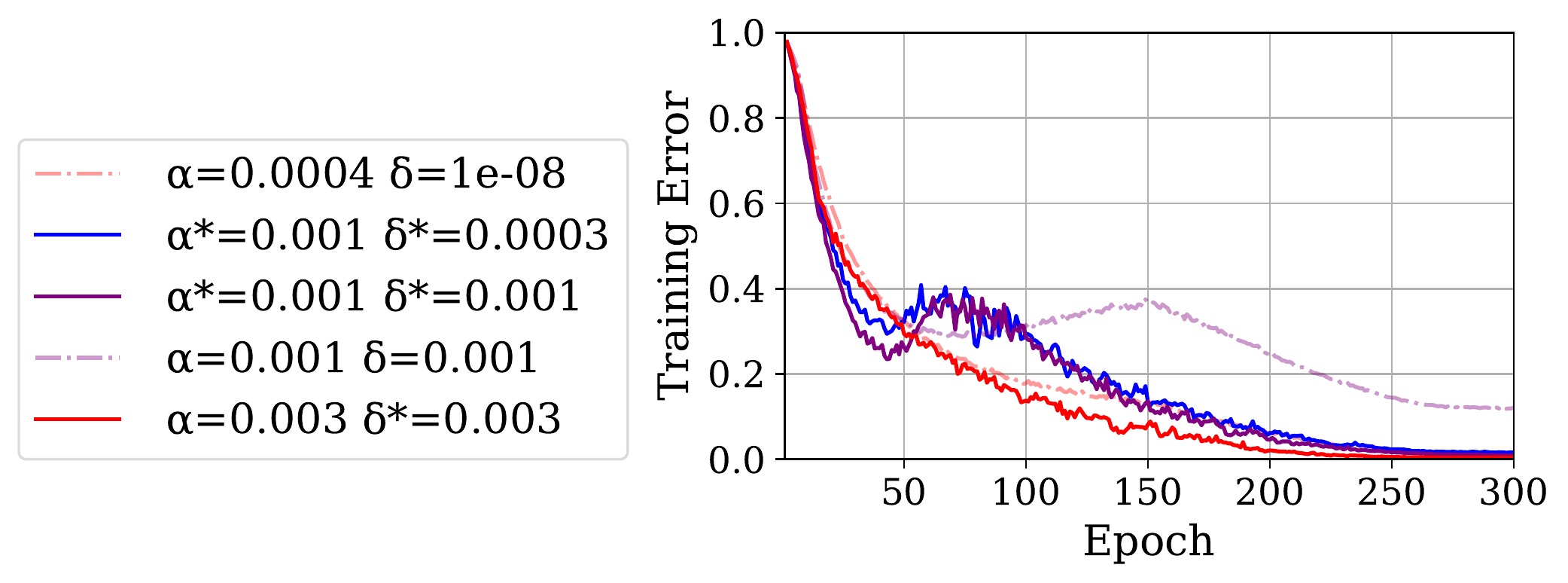}
		\caption{Training Error}
		\label{subfig:trainadam}
	\end{subfigure}
	\begin{subfigure}[b]{0.37\textwidth}
		\includegraphics[width=\textwidth]{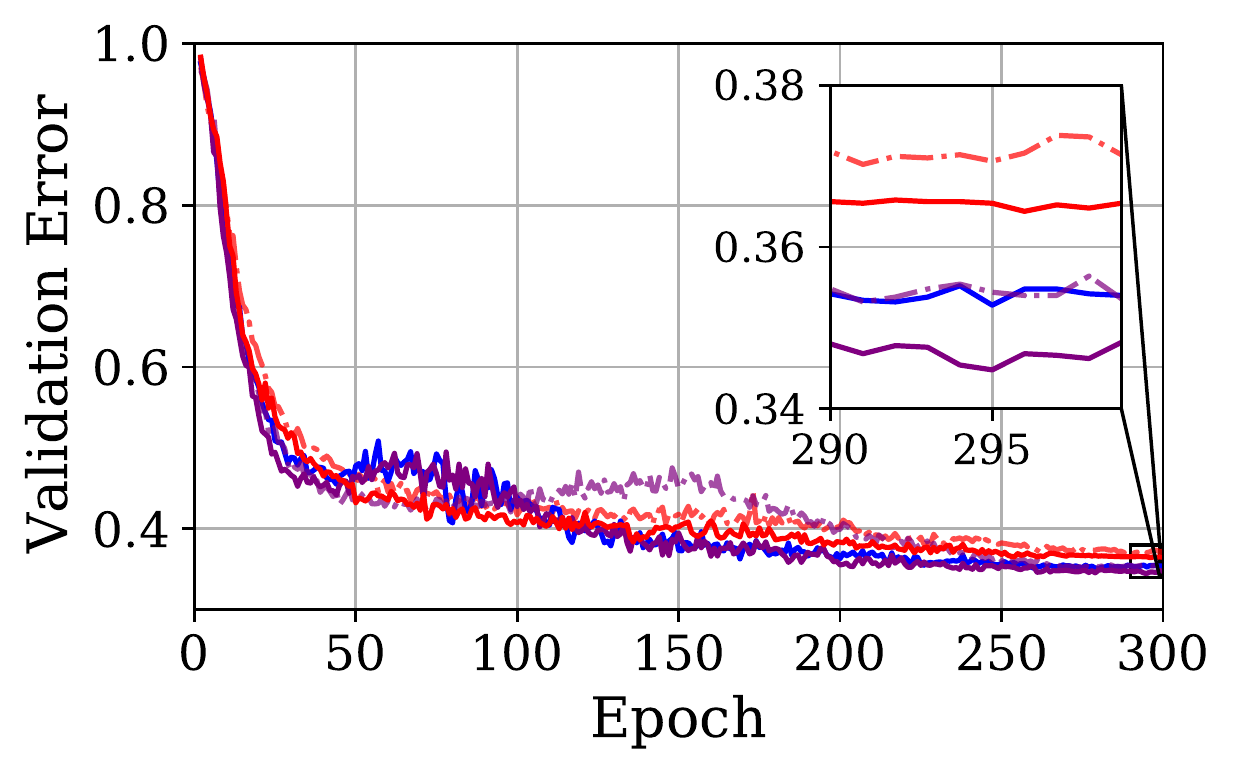}
		\caption{Val Error}
		\label{subfig:valadam}
	\end{subfigure}
	\caption{VGG-$16$ on CIFAR-$100$ dataset using the Adam optimiser with $\gamma=0$ (no weight decay) for a learning rate of $\alpha$, batch size $B=128$ and damping set by $\delta$. For adaptive damping methods the damping is given an initial floor value of $\delta^{*}$ and is then updated using the variance of the Hessian every $100$ steps. $\alpha^{*}$ refers to the use of an alternative ramp up schedule where the base learning rate is increased by a factor of $5$ at the start of training before being decreased.}
	\label{fig:adamautodamp}
\end{figure}
Whilst an increase in learning rate and damping, along with auto-damping improves both the convergence and validation result over the standard baseline (where the damping is kept at the default value and maximal learning rate is found which stably trains) the improvements are small and do not make up the gap with SGD. More specifically they are not better than just using a larger learning rate in combination with a larger flat damping, defeating the purpose of learning the damping factor online. 

To alleviate the effect of overly harsh damping, we consider an alternate learning rate schedule where the base learning rate is increased by a factor of $5$ early in training and then subsequently decreased. The constant $5$ is not tuned but simply a place-holder to consider a more aggressive learning rate schedule to counter-act the effect of the damping learner. These curves are marked with $\alpha^{*}$ in Figure \ref{fig:adamautodamp}. 

\paragraph{Warm up Learning Rate Schedule} For all experiments unless specified,  we use the following learning rate schedule for the learning rate at the $t$-th epoch:
\begin{equation}
	\alpha_t = 
	\begin{cases}
		\alpha_0, & \text{if}\ \frac{t}{T} \leq 0.1 \\
		\alpha_0[1+\frac{(\kappa-1)(\frac{t}{T}-0.1)}{0.2} , & \text{if}\ \frac{t}{T} \leq 0.3 \\
		\alpha_0[\kappa - \frac{(\kappa - r)(\frac{t}{T} - 0.3)}{0.6}] & \text{if } 0.3 < \frac{t}{T} \leq 0.9 \\
		\alpha_0r, & \text{otherwise}
	\end{cases}
\end{equation}
where $\alpha_0$ is the initial learning rate. $T$ is the total number of epochs budgeted for all CIFAR experiments. We set $r = 0.01$ and $\kappa = 5$.

While this introduces some slight training instability early in training, which could potentially be managed by altering the schedule, we find that such a schedule boosts the validation performance, particularly so for auto-damped methods, as shown by the blue curve in Figure \ref{subfig:valadam}, which surpasses the generalisation of SGD (shown in Figure \ref{fig:adamsgd}).

To more clearly expose the combined impact of adaptive damping and this alternative learning schedule we consider the variations in Figure \ref{fig:adamautodampclean} for a learning rate and damping both equal to $0.0001$. Here we see that the aggressive learning rate schedule with flat damping diverges, whereas the autodamping stabilises training, allowing for convergence to a solution with excellent generalisation. We see here in Figure \ref{subfig:closeadamdamp} that the damping coefficient reacts to this large learning rate increase by increasing its rate of damping early, stabilising training. We also show for reference that the typical linear decay schedule, with a larger learning rate and initial damping does not supersede the validation result of smaller learning rate and flat damping counter-part (it does however train better). This demonstrates the necessity of an alternative learning rate schedule to bring out the value of the adaptive damping. We remark however that optimal results in deep learning almost always require some degree of hand-crafted tuning of the learning rate. Our adaptive damping method is not proposed as a panacea, but just an optimal method of setting the damping coefficient. Since changing the damping coefficient effectively changes to geometry of the loss surface, it is entirely reasonable that the learning rate may have to be tweaked to give best results.

\begin{figure}[h!]
	\begin{subfigure}[b]{0.48\textwidth}
		\includegraphics[width=\textwidth]{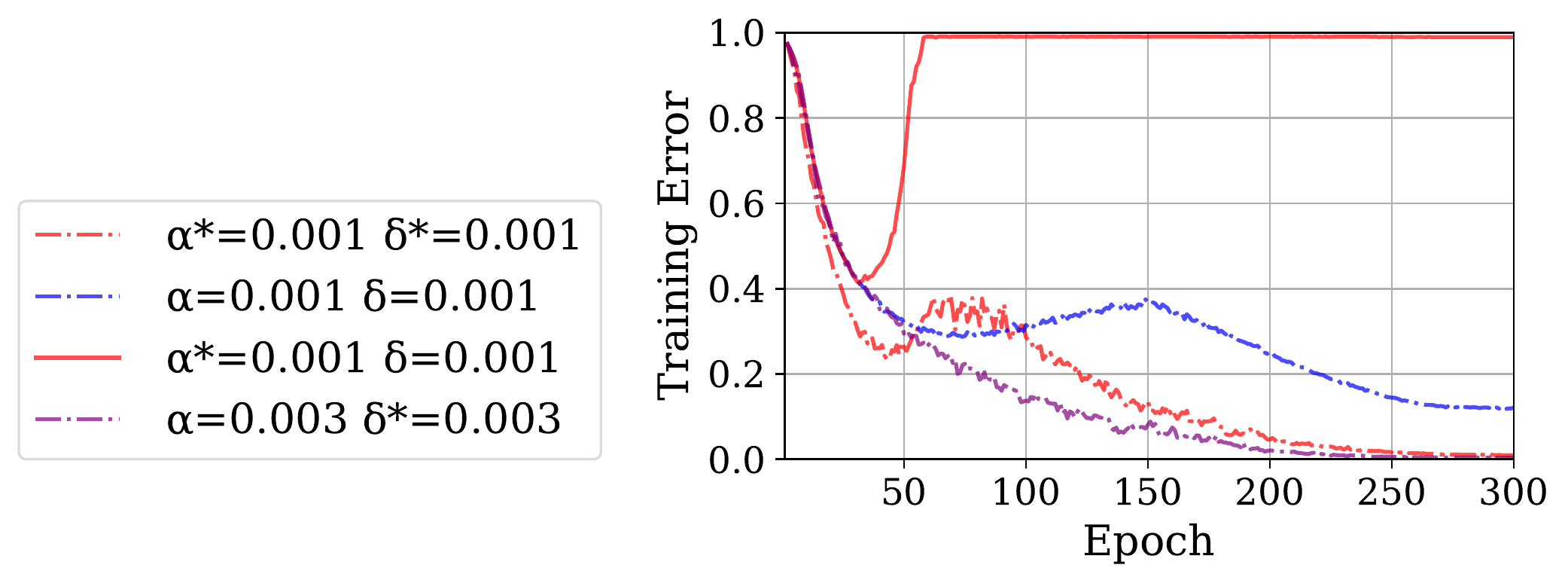}
		\caption{Training Error}
		\label{subfig:closetrainadam}
	\end{subfigure}
	\begin{subfigure}[b]{0.29\textwidth}
		\includegraphics[width=\textwidth]{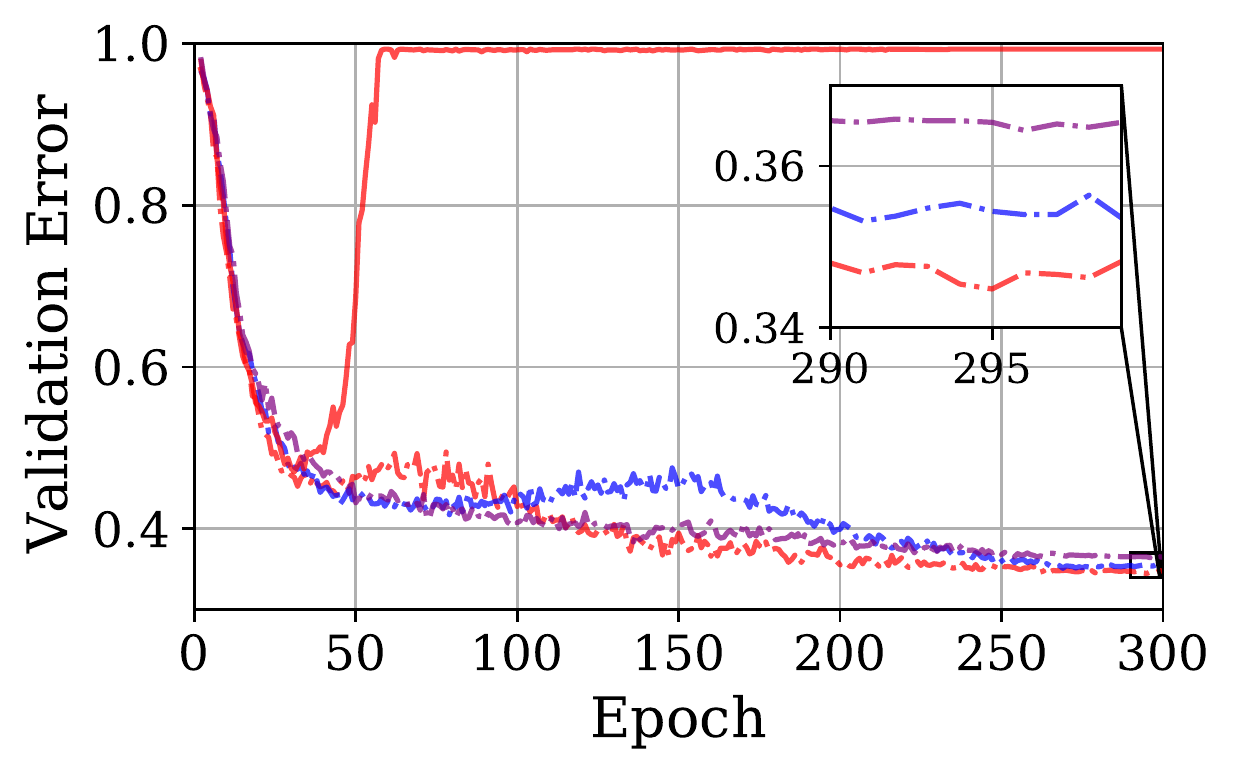}
		\caption{Val Error}
		\label{subfig:closevaladam}
	\end{subfigure}
	\begin{subfigure}[b]{0.205\textwidth}
		\includegraphics[width=\textwidth]{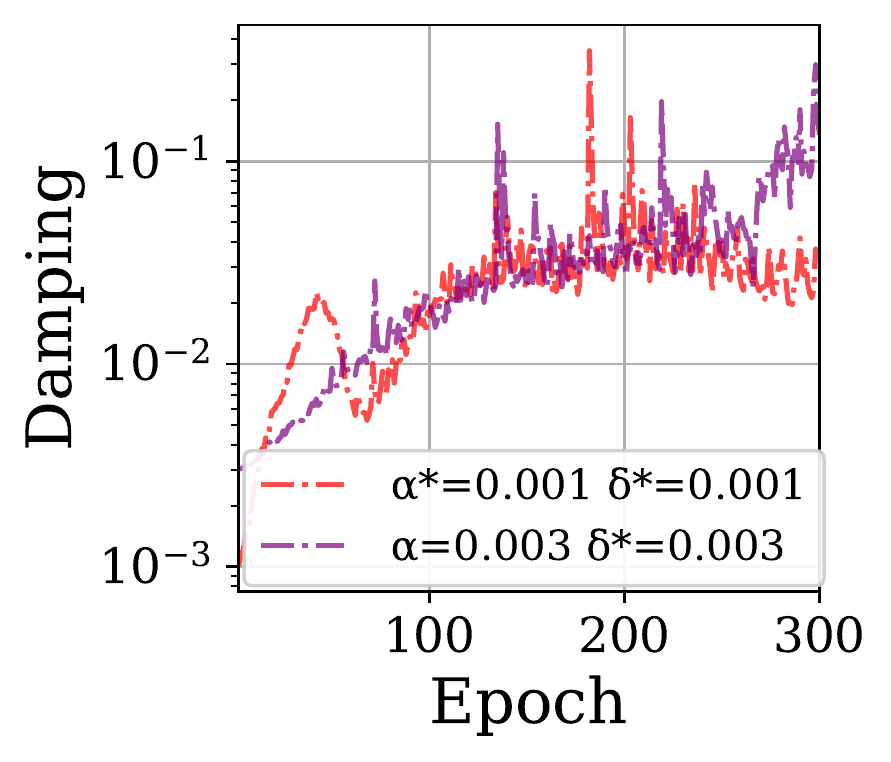}
		\caption{Damping}
		\label{subfig:closeadamdamp}
	\end{subfigure}
	\caption{VGG-$16$ on CIFAR-$100$ dataset using the Adam optimiser with $\gamma=0$ (no weight decay) for a learning rate of $\alpha$, batch size $B=128$ and damping set by $\delta$. For adaptive damping methods the damping is given an initial floor value of $\delta^{*}$ and is then updated using the variance of the Hessian every $100$ steps. $\alpha^{*}$ refers to the use of an alternative ramp up schedule where the base learning rate is increased by a factor of $5$ at the start of training before being decreased.}
	\label{fig:adamautodampclean}
\end{figure}

\section{Quality of adaptive Hessian approximations}\label{sec:quality}
% \subsection{How Closely do Adaptive Preconditioning Matrices Approximate the Batch Hessian?}\label{subsec:howclose}
We have seen so far the importance of sharp directions in the loss landscape, how adaptive methods bias away from movement in sharp directions and towards flat directions and we have seen how this issue can be addressed practically and in a principled manner using linear shrinkage theory. Behind all of this is the idea that using some approximation to the batch Hessian as a preconditioning matrix for gradients is helpful for optimisation. We have so far focused on mitigating the influence of noise in the batch Hessain compared to the true Hessian. Another natural question to ask is: how closely do adaptive preconditioning matrices approximate the batch Hessian?

\medskip
Adam uses the square root of the moving uncentered second moment of the per-parameter gradient \cite{kingma2014adam} that can be interpreted as a running diagonal approximation to the covariance of gradients, argued to be close to the Hessian for DNNs \cite{zhang2019algorithmic,jastrzebski2020the,fort2019stiffness,liu2019variance,he2019local}. Other methods employ Hessian--vector products, using either Lanczos or conjugate gradient techniques \cite{meurant2006lanczos} to approximate $\mH_{batch}^{-1}(\vw_{k})\nabla L_{batch}(\vw)$ \cite{martens2010deep,dauphin2014identifying}, which for a number of products $m \ll P$ can be seen as a low rank inverse approximation to $\mH_{batch}$, or direct inversion of Kronecker factored approximations \cite{martens2015optimizing}. 
\vspace{-7pt}
\paragraph{Are Adam and KFAC similar?} We validate the assertion that both Adam and KFAC approximate the batch Hessian to differing degrees of accuracy and hence fall under the same class of \textit{adaptive methods}. The batch loss to second order for any adaptive method is given by
\begin{equation}
\label{eq:ilikeitsharp}
	L(\vw-\alpha\mB^{-1}\nabla L) = L - \alpha (\nabla L)^T\mB^{-1/2}\bigg(1-\frac{\alpha}{2}\mB^{-1/2}\mH \mB^{-1/2}\bigg)\mB^{-1/2}\nabla L,
\end{equation} 
where we drop $\vw$ dependence and use the positive definiteness of $\mB$. This can derived using a Taylor expansion:
\begin{equation}
    L(\vw+\dw) = L(\vw)+\dw^{T}\nabla L(\vw) + \frac{1}{2}\dw^{T}\nabla \nabla L(\vw) \dw 
\end{equation}
Now substituting $\dw = \alpha \mB^{-1}\nabla L(\vw)$ and re-arranging we have
\begin{equation}
   \delta L =  L(\vw+\dw) - L(\vw) = \alpha (\nabla L(\vw))^{T}\mB^{-1}\nabla L(\vw) + \frac{\alpha^{2}}{2} (\nabla L(\vw))^{T}\mB^{-1}\nabla \nabla L(\vw) \mB^{-1}\nabla L(\vw)
   \nonumber
\end{equation}
from which the result (\ref{eq:ilikeitsharp}) follows immediately. From this expression we can derive conditions on the largest possible learning rate $\alpha$ for which the matrix on the right hand side of (\ref{eq:ilikeitsharp}) remains positive definite, namely $\alpha < 2( \| \mB^{-1/2}\mH \mB^{-1/2}\|_{\text{op}})^{-1}$. Now let $\{\eta_i, \vphi_i\}_{i=1}^P$ be the eigenvalues and normalised eigenvectors of $\mB$. We have \begin{align*}
    \| \mB^{-1/2}\mH \mB^{-1/2}\|_{\text{op}} = \max_i\{\eta^{-1}_i|\vphi_i^T\mH\vphi_i|\}.
\end{align*}
For SGD $\mB = I$ and so the largest stable $\alpha = 2(\| \mH\|_{\text{op}})^{-1}$, which we be expect to very small in practice. Now consider replacing $\eta_i$ by $\eta_i + \delta$. The largest stable learning rate is bounded above by \begin{align}\label{eq:quadratic_delta}
    2\left( \max_i \frac{|\vphi_i^T \mH \vphi_i|}{\eta_i + \delta}\right)^{-1}
\end{align}
so we again see that the very smallest $\eta_i$ can very much decrease the largest stable learning rate if $\delta$ is too small, but that increasing $\delta$ mollifies the effect of the very small $\eta_i$. 
% Re-inspecting (\ref{eq:ilikeitsharp}) we see that the largest stable \emph{effective} learning rate is bounded above by
% \begin{align*}
%     2( \| \mB^{-1/2}\mH \mB^{-1/2}\|_{\text{op}})^{-1} (\min_i \eta_i)^{-1} \leq (\max_i |\vphi_i^T \mH \vphi_i|)^{-1} \frac{\max_i \eta_i}{\min_i \eta_i}.
% \end{align*}
% So while increasing $\mB$ by an overall scaling increases the maximal stable $\alpha$, it does not increase the maximal stable effective learning rate. 

Now set the damping $\delta=1$ and note that \begin{align}
    \sum_i \vphi_i^T \mH \vphi_i = \Tr \mU^T\mH \mU \sum_i\ve_i\ve_i^T = \Tr H ~ \implies ~ P^{-1} \|\mH\|_{\text{op}}  \max_i |\vphi_i^T \mH \vphi_i| \leq \|\mH\|_{\text{op}}
\end{align}
for some orthogonal matrix $\mU$, where $(\ve_{i})_{j} = \delta_{ij}$. So let us write $\max_i |\vphi_i^T \mH \vphi_i| = \xi \|\mH\|_{\text{op}}$ where $\xi\in [P^{-1}, 1]$. However let us view $\mH$ as a random matrix with eigenvalues of typical size $\mathcal{O}(\sqrt{P})$, so its entries are typically of size $\mathcal{O}(1)$ with high probability. If the eigenvectors of $\mH$ are in general position\footnote{We expect this e.g. in the case that $\mB$ is diagonal (such as in Adam) and $\mX$ has delocalised eigenvectors.} relative to $\{\vphi_i\}$, then $\vphi_i \mH\phi_i$ are the just the diagonal entries of $\mH$ in a general basis and so with high probability $\vphi_i \mH\phi_i = \|\mH\|_{\text{op}}\xi + o(1)$ for non-zero $\xi$ with $|\xi| < 1$. At the other extreme, if $\{\vphi_i\}$ is the eigenbasis of $\mH$, then $\phi_i^T \mH \phi_i$ are just the eigenvalues of $\mH$ and so the same result holds. We therefore generically expect $\max_i |\vphi_i\mH\vphi_i|$ to roughly track the largest eigenvalue of $\mH$ and so the largest stable learning rate is determined by how well the largest eigenvalue/eigenvector pair of $\mH$ is estimated by $\mB$. Hence under this setup the largest learning rate achievable is a direct measure of the estimation accuracy of the sharpest eigenvalue/eigenvector pairs of the batch Hessian. We run SGD, Adam and KFAC on the VGG-$16$ architecture on the CIFAR-$100$ dataset with no weight decay for $300$ epochs using a linear annealed schedule.

We search for the highest stable learning rate $\alpha$ along a logarithmic grid, with end points $\in (0.01,1)$. All methods incorporate a momentum of $\rho = 0.9$. We find that the largest \emph{stable} rates for SGD, Adam and KFAC are $0.01,0.12,0.32$ respectively, indicating that both KFAC and Adam are significantly better able to estimate sharp directions than curvature-blind SGD. Moreover, the larger stable value of KFAC reflects that KFAC is better able to approximate the sharpest directions than Adam, which is restricted to diagonal forms of $\mB$.

\section{Conclusion}\label{sec:conclusions}
We show using a spiked random matrix model for the batch loss of deep neural networks that we expect sharp directions of loss surface to retain more information about the true loss surface compared to flatter directions. For adaptive methods, which attempt to minimise an implicit local quadratic of the sampled loss surface, this leads to sub-optimal steps with worse generalisation performance.
We further investigate the effect of damping on the solution sharpness and find that increasing damping always decreases the solution sharpness, linking to prior work in this area. We find that for large neural networks an increase in damping both assists training and is even able to best the SGD test baseline. An interesting consequence of this finding is that it suggests that damping should be considered an essential hyper-parameter in adaptive gradient methods as it already is in stochastic second order methods. Moreover, our random matrix theory model motivates a novel interpretation of damping as linear shrinkage estimation of the Hessian. We establish the validity of this interpretation by using shrinkage estimation theory to derive an optimal adaptive damping scheme which we show experimentally to dramatically improve optimisation speed with adaptive methods \emph{and} closes the adaptive generalisation gap.

\medskip
Our work leaves open several directions for further investigation and extension. Mathematically, there is the considerable challenge of determining optimal assumptions on the network, loss function and data distribution such that the key outlier overlap result in Theorem \ref{theorem:overlap}, or sufficiently similar analogues thereof, can be obtained. On the experimental side, we have restricted ourselves to computer vision datasets and a small number of appropriate standard network architectures. These choices helped to maintain clarity on the key points of investigation, however they are clearly limiting. In particular, it would be natural to reconsider our investigations in situations for which adaptive optimisers typically obtain state of the art results, such as modern natural language processing \cite{devlin2018bert}. Practically speaking, we have proposed a novel, theoretically motivated and effective adaptive damping method, but it is reliant on relatively expensive Hessian variance estimates throughout training. Future work could focus on cheaper methods of obtaining the required variance estimates. 
\medskip
\bibliographystyle{apalike}
\bibliography{gadam}
% \newpage

\appendix

\section{Convergence of the noise spectrum to the semi-circle}\label{sec:app_sc}
As discussed in the main text, the precise form of the overlap given in Theorem \ref{theorem:overlap} requires the Hessian noise matrix to have the semi-circle as its limiting spectral density. This can be established under sensible assumptions and we present the argument in this appendix. The key tenet is that there is sufficient (but not full) independence between the elements of the noise matrix. Hence for the additive model the noise matrix converges to the semi-circle law. We reproduce the technical conditions and key elements in the proof of \cite{granziol2020learning} below for clarity.

To derive analytic results, we employ the Kolmogorov limit \cite{bun2017cleaning}, where $P,B,N \rightarrow \infty$ but $\frac{P}{B} = q > 0$ and to account for dependence beyond the symmetry of the noise matrix elements, we introduce the $\sigma$-algebras $\mathfrak{F}^{(i,j)}$, and Lindeberg's ratio $L_{P}(\tau)$, which are defined for any $\tau>0$ as follows: 
\begin{equation}
\begin{aligned}
    & 	    \mathfrak{F}^{(i,j)} :=  \sigma \{ \meps(\vw)_{kl}:1\leq k \leq l \leq P, (k,l) \neq (i,j) \},\\
    &1 \leq i \leq j \leq P \\
    & L_{P}(\tau) := \frac{1}{P^{2}} \sum_{i,j=1}^{P}\mathbb{E}|\meps(\vw)_{i,j}|^{2}\bm{1}(|\meps(\vw)_{i,j}|\geq \tau \sqrt{P}). \\
\end{aligned}
\end{equation}

Where $\meps(\vw)$ defines the matrix of fluctuations, which denotes the difference matrix between the true and empirical Hessians i.e.
\begin{equation}
\meps(\vw) = \mH_{\mathrm{true}}(\vw)-\mH_{\mathrm{emp}}(\vw)
\end{equation}

Rewriting the fluctuation matrix as
$\meps(\vw) \equiv \mH_{batch}(\vw) - \mH_{emp}(\vw)$ and assuming the mini-batch to be drawn independently from the dataset, we can infer
\newline
\begin{equation}
\label{eq:noisematrix}
\begin{aligned}
     \meps(\vw) = & \bigg(\frac{1}{B}-\frac{1}{N}\bigg)\sum_{j=1}^{B}\nabla^{2} \ell(\vx_{j},\vw;\vy_{j}) \\
     & - \frac{1}{N}\sum_{i=B+1}^{N}\nabla^{2} \ell(\vx_{i},\vw;\vy_{i}) \\
\end{aligned}
\end{equation}
thus $\mathbb{E}(\meps(\vw)_{j,k}) = 0$ and \newline $\mathbb{E}(\meps(\vw)_{j,k})^{2} = \bigg(\frac{1}{B}-\frac{1}{N}\bigg)\mathrm{Var}[\nabla^{2} \ell(\vx,\vw;\vy)_{j,k}]$.  
\newline
Where $B$ is the batch size and $N$ the total dataset size. The expectation is taken with respect to the data generating distribution $\psi(\vx,\vy)$. In order for the variance in Equation \ref{eq:noisematrix} to exist, the elements of $\nabla^{2} \ell(\vw,\vw;\vy)$ must obey sufficient moment conditions. This can either be assumed as a technical condition, or alternatively derived under the more familiar condition of $L$-Lipschitz continuity, as shown with the following Lemma
\begin{lemma}
\label{lemma:boundedhessianelements}
For a Lipschitz-continuous empirical risk gradient and almost everywhere twice differentiable loss function $\ell(h(\vx;\vw),\vy)$, the elements of the fluctuation matrix $\meps(\vw)_{j,k}$ are strictly bounded in the range $- \sqrt{P}L \leq \meps(\vw)_{j,k}\leq \sqrt{P}L$. Where $P$ is the number of model parameters and $L$ is a constant.
\end{lemma}
\begin{proof}
As the gradient of the empirical risk is $L$ Lipschitz continous, as the empirical risk a sum over the samples, the gradient of the batch risk is also Lipschitz continous. As the difference of two Lipschitz functions is also Lipschitz, by the fundamental theorem of calculus and the definition of Lipschitz continuity the largest eigenvalue $\lambda_{max}$ of the fluctuation matrix $\meps(\vw)$ must be smaller than $L$. Hence using the Frobenius norm we can upper bound the matrix elements of $\meps(\vw)$
\begin{equation}
\begin{aligned}
 \text{Tr}(\meps(\vw)^{2})  = & \sum_{j,k=1}^{P}\meps(\vw)_{j,k}^{2}
= \meps(\vw)_{j=j',k=k'}^{2} \\ 
&+ \sum_{j\neq j',k\neq k'}^{P}\meps(\vw)_{j,k}^{2} = \sum_{i=1}^{P}\lambda_{i}^{2} \\
\end{aligned}
\end{equation}
thus $\meps(\vw)_{j=j',k=k'}^{2} \leq \sum_{i=1}^{P}\lambda_{i}^{2} \leq PL^{2}$ and $-\sqrt{P}L \leq \meps(\vw)_{j=j',k=k'} \leq \sqrt{P}L$.\\
\end{proof}
As the domain of the Hessian elements under the data generating distribution is bounded, the moments of Equation \ref{eq:noisematrix} are bounded and hence the variance exists. We can even go a step further with the following extra lemma.
\begin{lemma}
\label{lemma:normalelements}
For independent samples drawn from the data generating distribution and an $L$-Lipschitz loss $\ell$ the difference between the empirical Hessian and Batch Hessian converges element-wise to a zero mean, normal random variable with variance $\propto \frac{1}{B}-\frac{1}{N}$ for large $B,N$.
\end{lemma}
\begin{proof}
By Lemma \ref{lemma:boundedhessianelements}, the Hessian elements are bounded, hence the moments are bounded and using independence of samples and the central limit theorem, $(\frac{1}{B}-\frac{1}{N})^{-1/2}[\nabla^{2} R_{true}(\vw)-\nabla^{2} R_{emp}(\vw)]_{jk} \xrightarrow[a.s]{} \mathcal{N}(0,\sigma_{jk}^{2})$. 
\end{proof}

\begin{theorem}
\label{theorem:beastingassumptions}
The following technical conditions that in the limit $P \rightarrow \infty$, the limiting spectra density of $\meps(\vw)(\vw)$ is given by Wigner's semi-circle law \cite{akemann2011oxford} \newline $i) \frac{1}{P^{2}}\sum_{i,j=1}^{P}\mathbb{E}|\mathbb{E}(\meps(\vw)_{i,j}^{2}|\mathfrak{F}^{i,j})-\sigma^{2}_{i,j}| \rightarrow 0$, \newline $ii) L_{P}(\tau) \rightarrow 0$ for any $\tau>0$, \newline $iii) \frac{1}{P}\sum_{i}^{P}|\frac{1}{P}\sum_{j=1}^{P}\sigma_{i,j}^{2}-\sigma_{e}^{2}| \rightarrow 0$ \newline  $iv) \max_{1\leq i \leq P} \frac{1}{P}\sum_{j=1}^{P}\sigma_{i,j}^{2} \leq C$ 
\vspace{5pt}
\end{theorem}
\begin{proof}
Lindenberg's ratio is defined as $L_{P}(\tau) := \frac{1}{P^{2}} \sum_{i,j=1}^{P}\mathbb{E}|\meps(\vw)_{i,j}|^{2}\bm{1}(|\meps(\vw)_{i,j}|\geq \tau \sqrt{P})$.
By Lemma \ref{lemma:normalelements}, the tails of the normal distribution decay sufficiently rapidly such that $L_{P}(\tau) \rightarrow 0$ for any $\tau>0$ in the $P \rightarrow \infty$ limit. Alternatively, using the Frobenius identity and Lipschitz continuity $\sum_{i,j=1}^{P}\mathbb{E}|\meps(\vw)_{i,j}|^{2}\bm{1}(|\meps(\vw)_{i,j}|\geq \tau \sqrt{P}) \leq \sum_{i,j}^{P}\meps(\vw)_{i,j}^{2} = \sum_{i}^{P}\lambda_{i}^{2} \leq PL^{2}$, $L_{P}(\tau) \rightarrow 0$ for any $\tau>0$. 
\newline By Lemma \ref{lemma:normalelements} we also have $\mathbb{E}(\meps(\vw)_{i,j}|\mathfrak{F}^{i,j})=0$. Hence along with conditions $(i), (ii), (iii)$ the matrix $\meps(\vw)$ satisfies the conditions in \cite{gotze2012semicircle} and the and the limiting spectral density $p(\lambda)$ of $\meps(\vw) \in \mathbb{R}^{P\times P}$ converges to the semi circle law $p(\lambda) = \frac{\sqrt{4\sigma_{\epsilon}^{2}-\lambda^{2}}}{2\pi\sigma_{\epsilon}^{2}}$ 
\cite{gotze2012semicircle}. 
\newline \cite{gotze2012semicircle} use the condition $\frac{1}{P}\sum_{i=1}^{P}|\frac{1}{P}\sum_{j=1}^{P}\sigma_{i,j}^{2}-1| \rightarrow 0$, however this simply introduces a simple scaling factor, which is accounted for in condition $ii)$ and the corresponding variance per element of the limiting semi-circle.% which for reader clarity we explicitly derive in Appendix \ref{sec:backgroundtheory}
\end{proof}
We note that under the assumption of independence between all the elements of $\meps(\vw)$ we would have obtained the same result, as long as conditions $ii)$ and $iii)$ were obeyed. So in simple words, condition $9i)$ merely states that the dependence between the elements cannot be too large. For example completely dependent elements have a second moment expectation that scales as $P^{2}$ and hence condition $(i)$ cannot be satisfied. Condition $(ii)$ merely states that there cannot be too much variation in the variances per element and condition $(iii)$ that the variances are bounded. 

\end{document}